\documentclass{article}

\newcommand{\USETIKZ}{}

\usepackage[margin=2.5cm]{geometry}

\usepackage{amsmath}
\usepackage{amssymb}
\usepackage{amsthm}

\usepackage{algorithmic}
\usepackage[linesnumbered,titlenumbered,ruled,vlined,resetcount,algosection]{algorithm2e}
\usepackage{xpatch}
\xpretocmd{\algorithmic}{}{}{}

\usepackage[usenames,dvipsnames]{xcolor}
\usepackage{dcolumn,booktabs}
\usepackage{multirow}
\usepackage{float}

\usepackage[normalem]{ulem}

\usepackage{xparse} 
\usepackage{xifthen} 
\usepackage{ifmtarg} 

\usepackage[labelformat=simple]{subcaption}

\usepackage{todonotes}

\usepackage{hyperref}
\hypersetup{colorlinks=true,citecolor=darkblue, linkcolor=darkblue, urlcolor=darkblue}

\usepackage[natbibapa]{apacite}

\usepackage{pdflscape}

\usepackage{authblk}

\ifdefined\USETIKZ
\usepackage{tikz}
\fi

\makeatletter
\def\IfEmptyTF#1{%
  \if\relax\detokenize{#1}\relax
    \expandafter\@firstoftwo
  \else
    \expandafter\@secondoftwo
  \fi}
\makeatother

\newcolumntype{d}[1]{D{.}{.}{#1}}
\newcommand\mc[1]{\multicolumn{1}{c}{#1}} 

\newcommand \fixed {\diamond}

\newcommand \projective {\mathrm{pr}}
\newcommand \planar {\mathrm{pl}}

\NewDocumentCommand \geprob{m m m}			{\mathbb{P}_{#1}^{#2}\left(#3\right)}
\NewDocumentCommand \eprob{m m}				{\geprob{#1}{}{#2}}
\NewDocumentCommand \eprobfixed{m m m}		{\geprob{#1}{\fixed}{#2 \mid #3}}

\NewDocumentCommand \prob{m}				{\eprob{}{#1}}

\NewDocumentCommand \rprobfixed{m m}		{\eprobfixed{\projective}{#1}{#2}}

\NewDocumentCommand \lprob{m}				{\eprob{\planar}{#1}}

\NewDocumentCommand \geexpe{m m m}			{\mathbb{E}_{#1}^{#2}\left[#3\right]}
\NewDocumentCommand \geexpeapprox{m m m}	{\tilde{\mathbb{E}}_{#1}^{#2}\left[#3\right]}

\NewDocumentCommand \eexpe{m m}				{\geexpe{#1}{}{#2}}
\NewDocumentCommand \eexpeapprox{m m}		{\geexpeapprox{#1}{}{#2}}
\NewDocumentCommand \econdexpe{m m m}		{\geexpe{#1}{}{#2 \mid #3}}
\NewDocumentCommand \eexpefixed{m m m}		{\geexpe{#1}{\fixed}{#2 \mid #3}}

\NewDocumentCommand \expe{m}				{\eexpe{}{#1}}

\NewDocumentCommand \expefixed{m} 			{\geexpe{}{\fixed}{#1}}

\NewDocumentCommand \rexpe{m}				{\eexpe{\projective}{#1}}
\NewDocumentCommand \rexpeapprox{m}  		{\eexpeapprox{\projective}{#1}}
\NewDocumentCommand \rcondexpe{m m}			{\econdexpe{\projective}{#1}{#2}}
\NewDocumentCommand \rexpefixed{m m}		{\eexpefixed{\projective}{#1}{#2}}

\NewDocumentCommand \rexper{m m}			{\econdexpe{\projective}{#1}{#2}}
\NewDocumentCommand \rexperfixed{m m}		{\eexpefixed{\projective}{#1}{#2}}

\NewDocumentCommand \lexpe{m}				{\eexpe{\planar}{#1}}
\NewDocumentCommand \lexpeapprox{m}			{\eexpeapprox{\planar}{#1}}
\NewDocumentCommand \lexpefixed{m m}		{\eexpefixed{\planar}{#1}{#2}}
\NewDocumentCommand \lcondexpe{m m}			{\econdexpe{\planar}{#1}{#2}}

\NewDocumentCommand \var{m}  				{\mathbb{V}\left[#1\right]}

\NewDocumentCommand \bigO {m} {O(#1)}

\newtheoremstyle{theoremstyle}
    {15pt} 
    {} 
    {\itshape} 
    {} 
    {\bfseries} 
    {.} 
    {.5em} 
    {} 

\theoremstyle{theoremstyle}
\newtheorem{lemma}{Lemma}[section]
\newtheorem{theorem}[lemma]{Theorem}
\newtheorem{proposition}{Proposition}
\newtheorem{corollary}[proposition]{Corollary}

\NewDocumentCommand \Root {} {r}

\NewDocumentCommand \Ftree {O{T}} {#1}
\NewDocumentCommand \Rtree {O{\Root} O{\Ftree}} {#2^{#1}}
\NewDocumentCommand \SubRtree {m O{\Root} O{\Ftree}} {#3_{#1}^{#2}}

\NewDocumentCommand \Nvert {m m} {s_{#1}(#2)}

\NewDocumentCommand \degreesymbol { } {d}
\NewDocumentCommand \degree {m}				{\degreesymbol(#1)}
\NewDocumentCommand \outdegree {m O{\Root}}		{\degreesymbol_{#2}(#1)}

\NewDocumentCommand \neighbours {} {\Gamma}
\NewDocumentCommand \neighs {m}				{\neighbours(#1)}
\NewDocumentCommand \outneighs {m O{\Root}}		{\neighbours_{#2}(#1)}

\NewDocumentCommand \closest {m O{\Root}}	{c_{#2}(#1)}
\NewDocumentCommand \farthest {m O{\Root}}	{f_{#2}(#1)}

\NewDocumentCommand \arr {} {\pi}

\NewDocumentCommand \lenedge {m O{\arr}} {
	\IfEmptyTF{#2}
	{\delta_{#1}}
	{\delta_{#1}(#2)}
}
\NewDocumentCommand \lenedgep {m O{\arr}} { 
	\IfEmptyTF{#2}
	{\delta_{#1}^*}
	{\delta_{#1}^*(#2)}
}
\NewDocumentCommand \lengthsegment { m m O{\arr} } {
	\ifthenelse{\isempty{#3}}
	{\varphi_{#1}^{(#2)}}
	{\varphi_{#1,#3}(#2)}
}

\NewDocumentCommand \SetIndEdges {O{\Ftree}} {Q(#1)}
\NewDocumentCommand \D {m O{\arr}} {
	\IfEmptyTF{#1}
	{\IfEmptyTF{#2}
	{D}
	{D_{#2}}}
	{\IfEmptyTF{#2}
	{D(#1)}
	{D_{#2}(#1)}}
}
\NewDocumentCommand \Dp {m O{\arr}} {
	\IfEmptyTF{#2}
	{D^*(#1)}
	{D_{#2}^*(#1)}
}
\NewDocumentCommand \Cr {m O{\arr}} {C_{#2}(#1)}

\NewDocumentCommand \anchor { m O{\arr} } {
	\IfEmptyTF{#2}
	{\alpha_{#1}}
	{\alpha_{#1}(#2)}
}
\NewDocumentCommand \coanchor { m O{\arr} } {
	\IfEmptyTF{#2}
	{\beta_{#1}}
	{\beta_{#1}(#2)}
}

\NewDocumentCommand \gminD {m m}		{m_{#1}[D(#2)]}
\NewDocumentCommand \minDPr {O{\Rtree}}	{\gminD{\projective}{#1}}
\NewDocumentCommand \minDPl {O{\Ftree}}	{\gminD{\planar}{#1}}
\NewDocumentCommand \minD {O{\Ftree}}	{\gminD{}{#1}}

\NewDocumentCommand \Unc {O{\Ftree}}	{\mathbf{P}(#1)}
\NewDocumentCommand \NUnc {O{\Ftree}}	{\mathbf{N}(#1)}

\NewDocumentCommand \Proj {O{\Rtree}}		{\mathbf{P}_{\mathbf{pr}}(#1)}
\NewDocumentCommand \ProjFixed {m}			{\mathbf{P}^{\fixed}_{\mathbf{pr}}(\Rtree[#1])}
\NewDocumentCommand \NProj {O{\Rtree}}		{\mathbf{N}_{\mathbf{pr}}(#1)}
\NewDocumentCommand \NProjFixed {m}			{\mathbf{N}^{\fixed}_{\mathbf{pr}}(\Rtree[#1])}
\NewDocumentCommand \DProj {O{\Rtree}}		{\mathbf{D}_{\mathbf{pr}}(#1)}
\NewDocumentCommand \DProjFixed {O{\Rtree}}	{\mathbf{D}^{\fixed}_{\mathbf{pr}}(#1)}

\NewDocumentCommand \Plan {O{\Ftree}}	{\mathbf{P_{pl}}(#1)}
\NewDocumentCommand \NPlan {O{\Ftree}}	{\mathbf{N_{pl}}(#1)}
\NewDocumentCommand \DPlan {O{\Ftree}}	{\mathbf{D_{pl}}(#1)}

\NewDocumentCommand \Vd {m}		{\lenedge{#1}[]}

\NewDocumentCommand \VD {m}		{\D{#1}[]}

\NewDocumentCommand \VCr {m}		{\Cr{#1}[]}

\NewDocumentCommand \Vanchor {m}			{\anchor{#1}[]}

\NewDocumentCommand \Vcoanchor {m}			{\coanchor{#1}[]}

\NewDocumentCommand \Vlengthsegment {m m} 	{\lengthsegment{#1}{#2}[]}

\NewDocumentCommand \ExpeDUnc {O{\Ftree}}	{\expe{\VD{#1}}}
\NewDocumentCommand \ZExpeDUnc {O{\Ftree}}	{\Zexpe{\VD{#1}}}
\NewDocumentCommand \VarDUnc {O{\Ftree}}	{\var{\VD{#1}}}

\NewDocumentCommand \ExpeDProj {O{\Rtree}}			{\rexpe{\VD{#1}}}
\NewDocumentCommand \ExpeDProjApprox {O{\Rtree}}	{\rexpeapprox{\VD{#1}}}
\NewDocumentCommand \ExpeDProjFixed {m O{\Rtree}}	{\rexpefixed{\VD{#2}}{#1}}
\NewDocumentCommand \ExpeDProjFixedNoVertex {O{\Rtree}}	{\geexpe{\projective}{\fixed}{\VD{#1}}}
\NewDocumentCommand \condExpeDProj {m O{\Rtree}}	{\rcondexpe{\VD{#2}}{#1}}

\NewDocumentCommand \ExpeDPlan {O{\Ftree}}			{\lexpe{\VD{#1}}}
\NewDocumentCommand \ExpeDPlanApprox {O{\Ftree}}	{\lexpeapprox{\VD{#1}}}
\NewDocumentCommand \ExpeDPlanFixed {O{\Ftree}}		{\lexpefixed{\VD{#1}}}
\NewDocumentCommand \condExpeDPlan {m O{\Ftree}}	{\lcondexpe{\VD{#2}}{#1}}

\NewDocumentCommand \ExpeDk {O{k} O{\Ftree}}		{\eexpe{\le #1}{\VD{#2}}}
\NewDocumentCommand \ExpeDgek {O{k} O{\Ftree}}		{\eexpe{\ge #1}{\VD{#2}}}

\NewDocumentCommand \ExpeCUnc {O{\Ftree}}	{\expe{\VCr{#1}}}
\NewDocumentCommand \VarCUnc {O{\Ftree}}	{\var{\VCr{#1}}}

\NewDocumentCommand \deftree {} {t}

\NewDocumentCommand \deffilter {} {f}

\NewDocumentCommand \defomega {} {\omega}
\NewDocumentCommand \SymbolOmegaUnc {} {\Omega}
\NewDocumentCommand \OmegaUnc {O{\Ftree} O{\arr} O{\SymbolOmegaUnc}} {
    #3
	\IfEmptyTF{#1} {
	    \IfEmptyTF{#2} {
	    }{
	        \GenericError{}{Error: first parameter cannot be empty while the second isn't. Second parameter: '#2'.}{\@gobble}{}
	    }
	}{
	    \IfEmptyTF{#2} {
	        (#1)
	    }{
	        (#1;#2)
	    }
	}
}

\NewDocumentCommand \SymbolOmegaPlan {} {\Omega_{\planar}}
\NewDocumentCommand \OmegaPlan {O{\Ftree} O{\arr} O{\SymbolOmegaPlan}} {
    #3
	\IfEmptyTF{#1} {
	    \IfEmptyTF{#2} {
	    }{
	        \GenericError{}{Error: first parameter cannot be empty while the second isn't. Second parameter: '#2'.}{\@gobble}{}
	    }
	}{
	    \IfEmptyTF{#2} {
	        (#1)
	    }{
	        (#1;#2)
	    }
	}
}

\NewDocumentCommand \SymbolOmegaProj {} {\Omega_{\projective}}
\NewDocumentCommand \OmegaProj {O{\Rtree} O{\arr} O{\SymbolOmegaProj}} {
    #3
	\IfEmptyTF{#1} {
	    \IfEmptyTF{#2} {
	    }{
	        \text{(Error)}
	    }
	}{
	    \IfEmptyTF{#2} {
	        (#1)
	    }{
	        (#1;#2)
	    }
	}
}

\NewDocumentCommand \SymbolScoreOmega {O{\mu}} {{#1}_{\Omega}}
\NewDocumentCommand \ScoreOmega {m m m O{\mu}} {\SymbolScoreOmega[#4](#1,#2,#3)}
\NewDocumentCommand \ScoreOmegaF {O{\defomega} O{\deftree} O{\deffilter} O{\mu}} {\ScoreOmega{#1}{#2}{#3}[#4]}

\NewDocumentCommand \SymboldUnc {} {\bar{d}}
\NewDocumentCommand \dUnc {O{\Ftree} O{\arr} O{\SymboldUnc}} {
    #3
	\IfEmptyTF{#1} {
	    \IfEmptyTF{#2} {
	    }{
	        \GenericError{}{Error: first parameter cannot be empty while the second isn't. Second parameter: '#2'.}{\@gobble}{}
	    }
	}{
	    \IfEmptyTF{#2} {
	        (#1)
	    }{
	        (#1;#2)
	    }
	}
}

\NewDocumentCommand \SymbolScoreD {O{\mu} O{\SymboldUnc}} {{#1}_{\SymboldUnc}}
\NewDocumentCommand \ScoreD {m m O{\mu}} {\SymbolScoreD[#3](#1,#2)}
\NewDocumentCommand \ScoreDF {O{\deftree} O{\deffilter} O{\mu}} {\ScoreD{#1}{#2}[#3]}

\NewDocumentCommand \ExpedUnc {O{\Ftree}} {
	\IfEmptyTF{#1}
	{d}
	{d(#1)}
}
\NewDocumentCommand \ExpedProj {O{\Rtree}} {
	\IfEmptyTF{#1}
	{d_{\projective}}
	{d_{\projective}(#1)}
}
\NewDocumentCommand \ExpedPlan {O{\Ftree}} {
	\ifthenelse{\isempty{#1}}
	{d_{\planar}}
	{d_{\planar}(#1)}
}

\ifdefined\USETIKZ

\usetikzlibrary{arrows.meta,patterns}
\usetikzlibrary{ipe} 

\tikzstyle{ipe stylesheet} = [
  ipe import,
  even odd rule,
  line join=round,
  line cap=butt,
  ipe pen normal/.style={line width=0.4},
  ipe pen heavier/.style={line width=0.8},
  ipe pen fat/.style={line width=1.2},
  ipe pen ultrafat/.style={line width=2},
  ipe pen normal,
  ipe mark normal/.style={ipe mark scale=3},
  ipe mark large/.style={ipe mark scale=5},
  ipe mark small/.style={ipe mark scale=2},
  ipe mark tiny/.style={ipe mark scale=1.1},
  ipe mark normal,
  /pgf/arrow keys/.cd,
  ipe arrow normal/.style={scale=7},
  ipe arrow large/.style={scale=10},
  ipe arrow small/.style={scale=5},
  ipe arrow tiny/.style={scale=3},
  ipe arrow normal,
  /tikz/.cd,
  ipe arrows, 
  <->/.tip = ipe normal,
  ipe dash normal/.style={dash pattern=},
  ipe dash dotted/.style={dash pattern=on 1bp off 3bp},
  ipe dash dashed/.style={dash pattern=on 4bp off 4bp},
  ipe dash dash dotted/.style={dash pattern=on 4bp off 2bp on 1bp off 2bp},
  ipe dash dash dot dotted/.style={dash pattern=on 4bp off 2bp on 1bp off 2bp on 1bp off 2bp},
  ipe dash normal,
  ipe node/.append style={font=\normalsize},
  ipe stretch normal/.style={ipe node stretch=1},
  ipe stretch normal,
  ipe opacity 10/.style={opacity=0.1},
  ipe opacity 30/.style={opacity=0.3},
  ipe opacity 50/.style={opacity=0.5},
  ipe opacity 75/.style={opacity=0.75},
  ipe opacity opaque/.style={opacity=1},
  ipe opacity opaque,
]

\definecolor{red}{rgb}{1,0,0}
\definecolor{blue}{rgb}{0,0,1}
\definecolor{green}{rgb}{0,1,0}
\definecolor{yellow}{rgb}{1,1,0}
\definecolor{orange}{rgb}{1,0.647,0}
\definecolor{gold}{rgb}{1,0.843,0}
\definecolor{purple}{rgb}{0.627,0.125,0.941}
\definecolor{gray}{rgb}{0.745,0.745,0.745}
\definecolor{brown}{rgb}{0.647,0.165,0.165}
\definecolor{navy}{rgb}{0,0,0.502}
\definecolor{pink}{rgb}{1,0.753,0.796}
\definecolor{seagreen}{rgb}{0.18,0.545,0.341}
\definecolor{turquoise}{rgb}{0.251,0.878,0.816}
\definecolor{violet}{rgb}{0.933,0.51,0.933}
\definecolor{darkblue}{rgb}{0,0,0.545}
\definecolor{darkcyan}{rgb}{0,0.545,0.545}
\definecolor{darkgray}{rgb}{0.663,0.663,0.663}
\definecolor{darkgreen}{rgb}{0,0.392,0}
\definecolor{darkmagenta}{rgb}{0.545,0,0.545}
\definecolor{darkorange}{rgb}{1,0.549,0}
\definecolor{darkred}{rgb}{0.545,0,0}
\definecolor{lightblue}{rgb}{0.678,0.847,0.902}
\definecolor{lightcyan}{rgb}{0.878,1,1}
\definecolor{lightgray}{rgb}{0.827,0.827,0.827}
\definecolor{lightgreen}{rgb}{0.565,0.933,0.565}
\definecolor{lightyellow}{rgb}{1,1,0.878}
\definecolor{black}{rgb}{0,0,0}
\definecolor{white}{rgb}{1,1,1}

\fi

\begin{document}
\allowdisplaybreaks

\title{The expected sum of edge lengths in planar linearizations of trees. Theory and applications.}


\author[1]{Llu\'is Alemany-Puig}
\author[1]{Ramon Ferrer-i-Cancho}
\affil[1]{Universitat Polit\`ecnica de Catalunya (UPC), Barcelona, Catalonia, Spain.}





\maketitle


\begin{abstract}
Dependency trees have proven to be a very successful model to represent the syntactic structure of sentences of human languages. In these structures, vertices are words and edges connect syntactically-dependent words. The tendency of these dependencies to be short has been demonstrated using random baselines for the sum of the lengths of the edges or its variants. A ubiquitous baseline is the expected sum in projective orderings (wherein edges do not cross and the root word of the sentence is not covered by any edge), that can be computed in time $\bigO{n}$. Here we focus on a weaker formal constraint, namely planarity. In the theoretical domain, we present a characterization of planarity that, given a sentence, yields either the number of planar permutations or an efficient algorithm to generate uniformly random planar permutations of the words. We also show the relationship between the expected sum in planar arrangements and the expected sum in projective arrangements. In the domain of applications, we derive a $\bigO{n}$-time algorithm to calculate the expected value of the sum of edge lengths. We also apply this research to a parallel corpus and find that the gap between actual dependency distance and the random baseline reduces as the strength of the formal constraint on dependency structures increases, suggesting that formal constraints absorb part of the dependency distance minimization effect. Our research paves the way for replicating past research on dependency distance minimization using random planar linearizations as random baseline.
\end{abstract}

\section{Introduction}
\label{sec:introduction}

A successful representation of the structure of a sentence in natural language is a (labeled) graph indicating the syntactic relationships between words together with the encoding of the words' order. In such a graph, the edge labels indicate the type of syntactic relationship between the words. Such combination of graph and linear ordering, as in Figure \ref{fig:example_syntactic_dependency_graphs}, is known as syntactic dependency structure \citep{Nivre2006a}. When the graph is (1) {\em well-formed}, namely, the graph is weakly connected, (2) is {\em acyclic}, that is, there are no cycles in the graph, (3) is {\em single-headed}, that is, every node has a single head (except for the root node), and (4) there is only one root node (one node with no head) in the graph, then it is called a syntactic dependency tree \citep{Nivre2006a}. There exist {\em formal constraints} that are often imposed on dependency structures. One such constraint is {\em projectivity}: a dependency structure is projective if, for every vertex $v$, all vertices reachable from $v$ in the underlying graph form a continuous substring within the sentence \citep{Kuhlmann2006a} and the root word of the sentence (the root of the underlying syntactic dependency structure) is never covered (as in Figure \ref{fig:example_syntactic_dependency_graphs}(a)). Another formal constraint is {\em planarity}, a generalization of projectivity where the root is allowed to be covered by one or more of the edges (as in Figure \ref{fig:example_syntactic_dependency_graphs}(b)). Figure \ref{fig:example_syntactic_dependency_graphs}(c) shows a sentence that is neither projective nor planar.

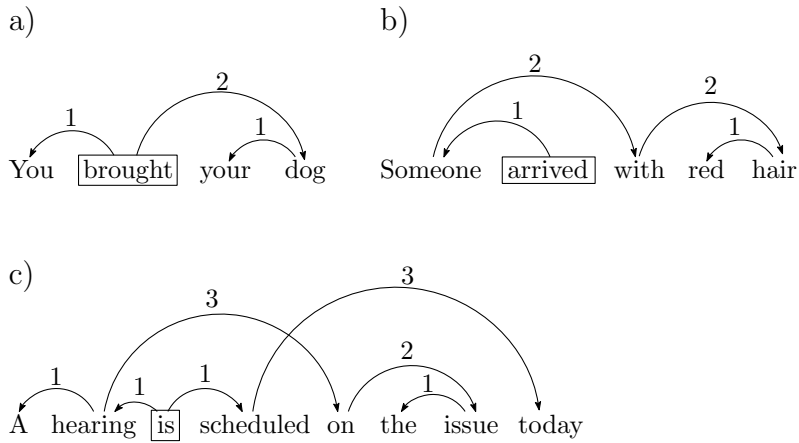
\begin{figure}
	\centering
\scalebox{1}{
\begin{tikzpicture}[ipe stylesheet]
  \node[ipe node, font=\large]
     at (40, 328) {a)};
  \node[ipe node, font=\large]
     at (180, 328) {b)};
  \node[ipe node]
     at (40, 176) {A};
  \node[ipe node]
     at (56, 176) {hearing};
  \node[ipe node]
     at (96, 176) {is};
  \node[ipe node]
     at (112, 176) {scheduled};
  \node[ipe node]
     at (160, 176) {on};
  \node[ipe node]
     at (180, 176) {the};
  \node[ipe node]
     at (204, 176) {issue};
  \node[ipe node]
     at (232, 176) {today};
  \draw[-{ipe pointed[ipe arrow tiny]}]
    (72, 184)
     arc[start angle=26.5651, end angle=153.4349, radius=15.6525];
  \draw[-{ipe pointed[ipe arrow tiny]}]
    (96, 184)
     arc[start angle=45, end angle=135, radius=11.3137];
  \draw[-{ipe pointed[ipe arrow tiny]}]
    (100, 184)
     arc[start angle=-153.4349, end angle=-26.5651, x radius=15.6525, y radius=-15.6525];
  \draw[-{ipe pointed[ipe arrow tiny]}]
    (76, 184)
     arc[start angle=-169.6952, end angle=-10.3048, x radius=44.7214, y radius=-44.7214];
  \draw[-{ipe pointed[ipe arrow tiny]}]
    (132, 184)
     arc[start angle=-171.8931, end angle=-8.7462, x radius=54.5451, y radius=-54.5451];
  \draw[-{ipe pointed[ipe arrow tiny]}]
    (168, 184)
     arc[start angle=-161.5651, end angle=-18.4349, x radius=25.2982, y radius=-25.2982];
  \draw[-{ipe pointed[ipe arrow tiny]}]
    (212, 184)
     arc[start angle=33.6901, end angle=146.3099, radius=14.4222];
  \node[ipe node]
     at (180, 272) {Someone};
  \node[ipe node]
     at (228, 272) {arrived};
  \node[ipe node]
     at (268, 272) {with};
  \node[ipe node]
     at (296, 272) {red};
  \node[ipe node]
     at (320, 272) {hair};
  \draw[-{ipe pointed[ipe arrow tiny]}]
    (244, 280)
     arc[start angle=21.8014, end angle=158.1986, radius=21.5407];
  \draw[-{ipe pointed[ipe arrow tiny]}]
    (199.9996, 280.0001)
     arc[start angle=-167.4725, end angle=-12.6643, x radius=39.0306, y radius=-39.0306];
  \draw[-{ipe pointed[ipe arrow tiny]}]
    (277.6044, 280.2802)
     arc[start angle=-163.5349, end angle=-18.4349, x radius=28.3178, y radius=-28.3178];
  \draw[-{ipe pointed[ipe arrow tiny]}]
    (328, 280)
     arc[start angle=33.6901, end angle=153.4349, radius=14.4222];
  \node[ipe node]
     at (40, 272) {You};
  \node[ipe node]
     at (68, 272) {brought};
  \node[ipe node]
     at (112, 272) {your};
  \node[ipe node]
     at (144, 272) {dog};
  \draw[-{ipe pointed[ipe arrow tiny]}]
    (79.6371, 280.6591)
     arc[start angle=28.952, end angle=153.4349, radius=17.8796];
  \draw[-{ipe pointed[ipe arrow tiny]}]
    (148, 280)
     arc[start angle=33.6901, end angle=153.4349, radius=14.4222];
  \draw[-{ipe pointed[ipe arrow tiny]}]
    (88.2365, 280.6029)
     arc[start angle=-164.5598, end angle=-12.5296, x radius=32.599, y radius=-32.599];
  \draw
    (226, 280) rectangle (261, 270);
  \draw
    (66.3206, 280.691) rectangle (103.487, 269.297);
  \node[ipe node]
     at (55.43, 194.119) {1};
  \node[ipe node]
     at (86.96, 190.239) {1};
  \node[ipe node]
     at (111.699, 195.09) {1};
  \node[ipe node]
     at (229.385, 294.455) {1};
  \node[ipe node]
     at (311.121, 288.149) {1};
  \node[ipe node]
     at (132.557, 288.118) {1};
  \node[ipe node]
     at (195.618, 191.451) {1};
  \node[ipe node]
     at (187.614, 203.821) {2};
  \node[ipe node]
     at (302.147, 303.671) {2};
  \node[ipe node]
     at (118.005, 305.096) {2};
  \node[ipe node]
     at (61.008, 291.756) {1};
  \node[ipe node]
     at (235.934, 312.16) {2};
  \node[ipe node]
     at (114.124, 223.224) {3};
  \node[ipe node]
     at (187.856, 231.956) {3};
  \node[ipe node, font=\large]
     at (40, 232) {c)};
  \draw
    (93.746, 184.453) rectangle (104.354, 173.013);
\end{tikzpicture}
}
	\caption{Examples of sentences with their syntactic dependency structures; arc labels indicate dependency distance (in words) between linked words. The rectangles denote the root word in each sentence. a) A projective dependency tree (adapted from \citealp{Gross2009a}). b) Planar (but not projective) syntactic dependency structure (adapted from \citealp{Gross2009a}). c) Non-projective and non-planar syntactic dependency structure (adapted from \citealp{Nivre2009a}).}
	\label{fig:example_syntactic_dependency_graphs}
\end{figure}

In this article, we study statistical properties of syntactic dependency structures under the planarity constraint. Such structures are represented in this article as a pair consisting of a (free or rooted) tree and a linear arrangement of its vertices. Free trees are denoted as $\Ftree=(V,E)$, and rooted trees as $\Rtree=(V,E;\Root)$, where $V$ is the set of vertices, $E$ the set of edges, and $\Root\in V$ denotes the root vertex. Unless stated otherwise $n=|V|$, that is, $n$ denotes the number of vertices which is equal to the number of words in the sentence. A linear arrangement $\arr$ (also called {\em embedding}) of a tree is a (bijective) function ($\arr \;:\; V \rightarrow \{1,\dots,n\}$) that maps every vertex $u$ of a tree to a unique position in $\{1,\dots,n\}$, which is denoted by $\arr(u)$.

Projectivity, as well as planarity, can be alternatively defined on linear arrangements using the concept of edge crossing. We say that any two (undirected) edges $\{s,t\}$, $\{u,v\}$ cross if the positions of their vertices interleave. More formally, assume, without loss of generality, that $\arr(s)<\arr(t)$, $\arr(u)<\arr(v)$ and $\arr(s)<\arr(u)$. Then, edges $\{s,t\}$, $\{u,v\}$ cross in the linear ordering defined by $\arr$ if $\arr(s)<\arr(u)<\arr(t)<\arr(v)$.\footnote{ Notice that this notion of crossing does not depend on edge orientation.} We denote the total number of edge crossings in an arrangement $\arr$ as $\Cr{\Ftree}$. Then, an arrangement $\arr$ of a rooted tree $\Rtree$ is {\em planar} if $\Cr{\Rtree}=0$ and is {\em projective} if (a) it is planar and (b) the root of the tree is not covered, that is, there is no edge $\{s,t\}$ such that $\arr(s)<\arr(\Root)<\arr(t)$ or $\arr(t)<\arr(\Root)<\arr(s)$. Planarity is a relaxation of projectivity where the root can be covered \citep{Sleator1993a,Kuhlmann2006a}. Planar arrangements are also known in the literature as {\em one-page book embeddings} \citep{Bernhart1974a}.

In this article, the main object of study is the expectation of the sum of edge lengths (or syntactic dependency distances) in planar arrangements of free trees. The length of an edge connecting two syntactically-related words, also known as dependency distance, is usually\footnote{Another popular definition is $\lenedge{uv}=|\arr(u) - \arr(v)| - 1$ \citep{Liu2017a}.} defined as the number of intervening words between $u$ and $v$ in the sentence plus 1 (Figure \ref{fig:example_syntactic_dependency_graphs}). It is defined mathematically as
\begin{equation*}
\lenedge{uv} = |\arr(u) - \arr(v)|.
\end{equation*}
We define the total sum of edge lengths in $\arr$ as
\begin{equation}
\label{eq:sum_edge_lengths}
\D{\Ftree} = \sum_{uv\in E} \lenedge{uv}.
\end{equation}
Close attention has been paid to this metric in modern linguistic research since its causal relationship with cognitive cost was first put forward, to the best of our knowledge, by \citet{Hudson1995a}. The main causal argument is that the longer the dependency, the greater the memory burden arising from decay of activation and interference \citep{Hudson1995a,Liu2017a}. A number of studies have exposed the general tendency in languages to reduce $D$, the total sum of edge lengths, a reflection of a potentially universal cognitive force known as the Dependency Distance Minimization principle (DDm) \citep{Ferrer2004a,Liu2008a,Futrell2015a,Liu2017a,Ferrer2022a}. As an example of such cognitive cost, consider the sentences in Figures \ref{fig:example_cognitive_cost}(a) and \ref{fig:example_cognitive_cost}(b): it is not surprising that the latter is preferred over the former due to smaller total sum of edge lengths \citep{Morrill2000a}, the former's being $D=18$ and the latter's being $D=12$.

\begin{figure}
	\centering
\scalebox{1}{
\begin{tikzpicture}[ipe stylesheet]
  \node[ipe node]
     at (72, 548) {John};
  \node[ipe node]
     at (100, 548) { gave};
  \node[ipe node]
     at (124, 548) { the};
  \node[ipe node]
     at (144, 548) { painting};
  \node[ipe node]
     at (184, 548) { that};
  \node[ipe node]
     at (208, 548) { Mary};
  \node[ipe node]
     at (236, 548) { hated};
  \node[ipe node]
     at (268, 548) { to};
  \node[ipe node]
     at (284, 548) { Bill};
  \draw[-{ipe pointed[ipe arrow tiny]}]
    (108.0002, 556.0001)
     arc[start angle=23.7856, end angle=156.2161, x radius=13.8867, y radius=9.9179];
  \draw[-{ipe pointed[ipe arrow tiny]}]
    (108, 556)
     arc[start angle=11.7505, end angle=168.2495, x radius=-27.5013, y radius=19.6415];
  \draw[-{ipe pointed[ipe arrow tiny]}]
    (108.0002, 556)
     arc[start angle=3.4897, end angle=176.5102, x radius=-92.0104, y radius=65.7141];
  \draw[-{ipe pointed[ipe arrow tiny]}]
    (159.9997, 555.9999)
     arc[start angle=21.0648, end angle=158.9352, x radius=15.582, y radius=11.1287];
  \draw[-{ipe pointed[ipe arrow tiny]}]
    (164, 556)
     arc[start angle=7.254, end angle=172.746, x radius=-44.355, y radius=31.6785];
  \draw[-{ipe pointed[ipe arrow tiny]}]
    (248.176, 556.0009)
     arc[start angle=11.5049, end angle=168.4947, x radius=28.0863, y radius=20.0593];
  \draw[-{ipe pointed[ipe arrow tiny]}]
    (248.1767, 556.0009)
     arc[start angle=21.4339, end angle=158.5661, x radius=15.3299, y radius=10.9486];
  \draw[-{ipe pointed[ipe arrow tiny]}]
    (287.9999, 556)
     arc[start angle=35.7284, end angle=144.2716, x radius=9.591, y radius=6.8499];
  \node[ipe node]
     at (72, 472) {John};
  \node[ipe node]
     at (100, 472) { gave};
  \node[ipe node]
     at (124, 472) {Bill};
  \node[ipe node]
     at (144, 472) { the};
  \node[ipe node]
     at (164, 472) { painting};
  \node[ipe node]
     at (204, 472) { that};
  \node[ipe node]
     at (228, 472) { Mary};
  \node[ipe node]
     at (256, 472) { hated};
  \draw[-{ipe pointed[ipe arrow tiny]}]
    (109.5474, 480.0008)
     arc[start angle=23.6799, end angle=156.3254, x radius=13.9478, y radius=9.9616];
  \draw[-{ipe pointed[ipe arrow tiny]}]
    (109.5484, 480.0008)
     arc[start angle=26.635, end angle=153.363, x radius=-12.4956, y radius=8.9244];
  \draw[-{ipe pointed[ipe arrow tiny]}]
    (109.5478, 480.001)
     arc[start angle=8.8086, end angle=171.1916, x radius=-36.5827, y radius=26.1275];
  \draw[-{ipe pointed[ipe arrow tiny]}]
    (179.9997, 479.9999)
     arc[start angle=21.0648, end angle=158.9352, x radius=15.582, y radius=11.1287];
  \draw[-{ipe pointed[ipe arrow tiny]}]
    (184, 480)
     arc[start angle=7.254, end angle=172.746, x radius=-44.355, y radius=31.6785];
  \draw[-{ipe pointed[ipe arrow tiny]}]
    (268.176, 480.0009)
     arc[start angle=11.5049, end angle=168.4947, x radius=28.0863, y radius=20.0593];
  \draw[-{ipe pointed[ipe arrow tiny]}]
    (268.1767, 480.0009)
     arc[start angle=21.4339, end angle=158.5661, x radius=15.3299, y radius=10.9486];
  \draw
    (97.783, 556.021)
     -- (97.7541, 544.638)
     -- (122.6359, 544.602)
     -- (122.7894, 555.891)
     -- cycle;
  \draw
    (97.715, 479.967)
     -- (97.754, 468.638)
     -- (122.459, 468.612)
     -- (122.452, 479.967)
     -- cycle;
  \node[ipe node]
     at (72, 624) {a)};
  \node[ipe node]
     at (72, 516) {b)};
  \node[ipe node]
     at (92.527, 564.394) {1};
  \node[ipe node]
     at (135.249, 562.476) {1};
  \node[ipe node]
     at (132.086, 573.416) {2};
  \node[ipe node]
     at (222.769, 563.276) {1};
  \node[ipe node]
     at (206.005, 573.096) {2};
  \node[ipe node]
     at (194.127, 585.143) {3};
  \node[ipe node]
     at (277.648, 560.554) {1};
  \node[ipe node]
     at (154.124, 611.224) {7};
  \node[ipe node]
     at (91.43, 490.119) {1};
  \node[ipe node]
     at (123.43, 486.119) {1};
  \node[ipe node]
     at (155.43, 486.119) {1};
  \node[ipe node]
     at (242.79, 486.44) {1};
  \node[ipe node]
     at (226.974, 496.621) {2};
  \node[ipe node]
     at (207.856, 507.956) {3};
  \node[ipe node]
     at (143.696, 505.394) {3};
\end{tikzpicture}
}
	\caption{Examples of sentences with their syntactic dependency structures; arc labels indicate dependency distance. The rectangles denote the root word in each sentence. Examples adapted from \citet{Morrill2000a}. The sum of edge lengths are $D=18$ for a) and $D=12$ for b).}
	\label{fig:example_cognitive_cost}
\end{figure}
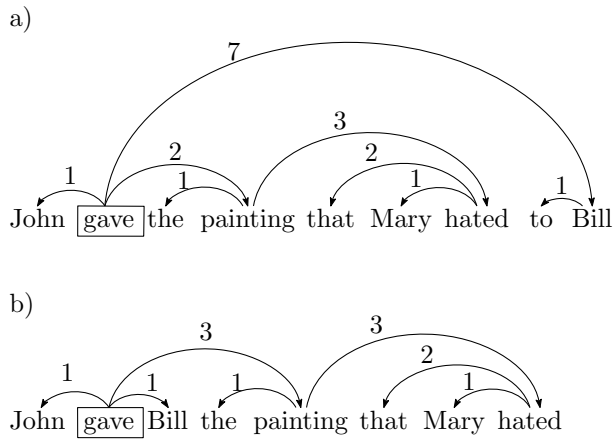

Statistical evidence of the DDm principle has been provided showing that dependency distances are smaller than expected by chance in syntactic dependency treebanks \citep{Ferrer2004a,Liu2008a,Park2009a,Gildea2010a,Futrell2015a,Liu2017a,Ferrer2022a,Kramer2021a}. Typically, the random baseline is defined as a random shuffling of the words of a sentence. To the best of our knowledge, the first known instance of such an approach was done by \citet{Ferrer2004a} who established the DDm principle by comparing the average real $\VD{\Ftree}$ of sentences against its expected value in a uniformly random permutation of their words. More formally, \citet{Ferrer2004a} calculated the expected value of $\VD{\Ftree}$ when the words of the sentence are shuffled uniformly at random (u.a.r.), that is, when all $n!$ permutations equally likely. This value is denoted here as $\ExpeDUnc$. \citet{Ferrer2004a} found that
\begin{equation}
\label{eq:introduction:E_D}
\ExpeDUnc = \frac{n^2 - 1}{3}.
\end{equation}
In spite of the simplicity of Equation \ref{eq:introduction:E_D}, the majority of researchers have used as random baseline the expected sum of edge lengths conditioned to projective arrangements \citep{Temperley2008a,Park2009a,Gildea2010a,Futrell2015a,Kramer2021a} which we denote here as $\ExpeDProj$. However, this baseline has been computed approximately via random sampling of projective arrangements. For these reasons, a formula to calculate the exact value of $\ExpeDProj$ in linear time, was derived by \citet{Alemany2022b}
\begin{equation}
\label{eq:introduction:E_pr_D}
\ExpeDProj
	= \frac{1}{6}\sum_{u\in V} \Nvert{\Root}{u}( 2\outdegree{u} + 1 ) - \frac{1}{6},
\end{equation}
where $\Nvert{\Root}{u}$ denotes the size (in vertices) of the subtree of $\Rtree$ rooted at $u$, and $\outdegree{u}$ is the out-degree of $u$ in $\Rtree$. In spite of its extensive use, the projective random baseline has some limitations. First, the percentage of non-projective sentences in languages ranges between $18.2$ and $26.4$ \citep{Gomez2016a} or between 6.8 and 36.4 \citep{Gomez2010a} (see also \citealp{Havelka2007a}). The limited coverage of projectivity raises the question if the projective baseline should be used for sentences that are not projective as it is customary in research on dependency distance minimization. In addition, projectivity {\em per se} implies a reduction in dependency distances, which raises the question if that rather strong constraint may mask the effect of the dependency distance minimization principle under investigation \citep{Gomez2022a}. Here we aim to make a step forward by considering planarity, a generalization of projectivity, so as to increase the coverage of real sentences and reduce the bias towards dependency minimization in the random baseline. The percentage of non-planar sentences in languages ranges between $14.3$ and $20.0$ \citep{Ferrer2018a} or between $5.3$ and $31$ \citep{Gomez2010a}. The latter range is consistent with earlier estimates \citep{Havelka2007a}.


This article is part of a research program on the statistical properties of $\VD{\Ftree}$ under constraints on the possible linear arrangements \citep{Ferrer2019a,Alemany2022a,Alemany2022b}. The remainder of the article is divided into two main parts: theory (Section \ref{sec:theory}) and applications (Section \ref{sec:applications}).

The theory part (Section \ref{sec:theory}) is structured as follows. In Section \ref{sec:theory:preliminaries}, we introduce notation used throughout that part. In Section \ref{sec:theory:counting_arrangements}, we first present a characterization of planar arrangements so as to identify their underlying structure, which we apply to count their number for a given free tree, and later on in Section  \ref{sec:theory:generating_arrangements}, to generate them u.a.r. by means of a novel $O(n)$-time algorithm. In Section \ref{sec:theory:E_pl_D}, we use said characterization to prove the main result of the article, namely that expectation of $\VD{\Ftree}$ in planar arrangements can be calculated from the expectation of projective arrangements, as the following theorem indicates.
\begin{theorem}
\label{thm:introduction:E_pl_D}
Given a free tree $\Ftree=(V,E)$,
\begin{align}
\ExpeDPlan
	\label{eq:introduction:E_pl_D:function_E_pr_D_fixed}
	&= \frac{1}{n}\sum_{u\in V} \ExpeDProjFixedNoVertex[\Rtree[u]] \\
	\label{eq:introduction:E_pl_D:function_E_pr_D}
	&= \frac{(n-1)(n-2)}{6n} + \frac{1}{n}\sum_{u\in V}\ExpeDProj[\Rtree[u]],
\end{align}
where $\ExpeDProjFixedNoVertex[\Rtree[u]]$ is the expected value of $\VD{\Rtree[u]}$ in uniformly random projective arrangements $\arr$ of $\Rtree[u]$ such that $\arr(u)=1$ and $\ExpeDProj[\Rtree[u]]$ (Equation \ref{eq:introduction:E_pr_D}) is the expected value of $\VD{\Rtree[u]}$ in uniformly random projective arrangements of $\Rtree[u]$, the free tree $\Ftree$ rooted at $u$.
\end{theorem}
Table \ref{table:introduction:results_summary} summarizes the theoretical results obtained in previous articles and those presented in this article.

The applications part (Section \ref{sec:applications}) is structured as follows. In Section \ref{sec:applications:algorithm}, we apply Theorem \ref{thm:introduction:E_pl_D} to derive a $\bigO{n}$-time algorithm to calculate $\ExpeDPlan$. Since \citet{Alemany2022b} showed that $\ExpeDProj$ can be evaluated in time $\bigO{n}$, Equation \ref{eq:introduction:E_pl_D:function_E_pr_D} naturally leads to a $\bigO{n^2}$-time algorithm if it is evaluated `as is'. However, we devise a $\bigO{n}$-time algorithm to calculate $\ExpeDPlan$. In Section \ref{sec:applications:corpora}, we apply this and previous research on the projective case \citep{Alemany2022b} to a parallel syntactic dependency treebank. We find that the gap between the actual dependency distance and that of the random baseline, reduces as the strength of the formal constraint on dependency structures chosen for the random baseline increases, suggesting that formal constraints absorb part of the dependency distance minimization effect.

Finally, in Section \ref{sec:conclusions}, we review all the findings and make suggestions for future research.

From this point onwards, the article is organized to ease reading by readers of distinct profiles. Readers interested in the analysis of syntactic dependency treebanks can jump directly to Section \ref{sec:applications:corpora}. Readers interested in the algorithm for computing $\ExpeDPlan$ can jump directly to Section \ref{sec:applications:algorithm}, after reading Section \ref{sec:theory:preliminaries}. Readers whose primary interest is applying the algorithms have ready-to-use code: both methods to generate planar arrangements (Section \ref{sec:theory:generating_arrangements}) and the $\bigO{n}$-time calculation of $\ExpeDPlan$ (Section \ref{sec:applications:algorithm}) are freely available in the Linear Arrangement Library\footnote{Available online at \url{https://github.com/LAL-project/linear-arrangement-library/}.} \citep{Alemany2021a}. 

\begin{table}
	\centering
	\begin{tabular}{crc}
	\toprule
	
	\multirow{3}*{Unconstrained ($\Ftree$)}
						& $\NUnc$			&	$n!$ \\
						& $\expe{\Vd{uv}}$	&	\( \displaystyle \frac{n + 1}{3} \) \\
						& $\ExpeDUnc$		&	\( \displaystyle \frac{n^2 - 1}{3} \) \\
	
	\midrule
	
	\multirow{3}*{Planar ($\Ftree$)}
						& $\NPlan$			&	\( \displaystyle n\prod_{u\in V} \degree{u}! \) \\
						& $\lexpe{\Vd{uv}}$	&	\( \displaystyle 1 + \frac{1}{n}\sum_{s\in V\setminus\{u,v\}} \rexper{\Vd{uv}}{s} \) \\
						& $\ExpeDPlan$		&	\( \displaystyle \frac{(n-1)(n-2)}{6n} + \frac{1}{n}\sum_{u\in V}\ExpeDProj[\Rtree[u]] \) \\
	
	\midrule
	
	\multirow{3}*{Projective ($\Rtree$)}
						& $\NProj$			&	\( \displaystyle \prod_{u\in V} (\outdegree{u} + 1)! \) \\
						& $\rexpe{\Vd{uv}}$	&	\( \displaystyle \frac{1}{6}(2\Nvert{\Root}{u} + \Nvert{\Root}{v} + 1) \) \\
						& $\ExpeDProj$		&	\( \displaystyle \frac{1}{6}\left(-1 + \sum_{v\in V}\Nvert{\Root}{v}(2\outdegree{v} + 1) \right) \) \\
												
	\bottomrule
	\end{tabular}
	
	\caption{Summary of the main mathematical results for increasing constraints on linear orders. Results for the unconstrained and projective cases are borrowed from previous research (\citealp{Ferrer2004a} and \citealp{Alemany2022b}, respectively). Results for the planar case are a contribution of the present article. $\NProj$, $\NPlan$ and $\NUnc$ denote the number of distinct projective, planar and unconstrained linear arrangements, respectively, of a rooted tree $\Rtree$ or of a free tree $\Ftree$. $\rexpe{\Vd{uv}}$, $\lexpe{\Vd{uv}}$ and $\expe{\Vd{uv}}$ denote the expected length of an edge in random linear arrangement for the projective, planar and unconstrained cases, respectively. $\rexper{\Vd{uv}}{s}$ is the expected value of $\Vd{uv}$ conditioned to having vertex $s$ as root of the tree. In $\rexpe{\Vd{uv}}$ the root is vertex $\Root$. }
	\label{table:introduction:results_summary}
\end{table}
\section{Theory}
\label{sec:theory}

\subsection{Definitions and notation}
\label{sec:theory:preliminaries}

We use $u,v,w,z$ to denote vertices, $\Root$ to always denote a root vertex, and $i,j,k,p,q$ to denote integers. The edges of a free tree are undirected, and denoted as $\{u,v\}=uv$; those of rooted trees are directed, denoted as $(u,v)$, and oriented away from $\Root$ towards the leaves. 

Let $\neighs{u}$ denote the set of neighbors of $u\in V$ in the free tree $\Ftree$, and let $\outneighs{u}$ denote the out neighbors (also, children) of $u\in V$ in $\Rtree$. Notice that, $\outneighs{u} \subseteq \neighs{u}$ with equality if, and only if $u=\Root$. Let $\outdegree{u}=|\outneighs{u}|$ denote the out-degree of vertex $u$ of a rooted tree $\Rtree$, and let $\degree{u}=|\neighs{u}|$ denote the degree of $u$ in a free tree $\Ftree$. Notice that $\outdegree{u}=\degree{u}-1$ when $u\neq\Root$ and $\outdegree{\Root}=\degree{\Root}$. Furthermore, we denote the subtree rooted at $v$ with respect to root $u$ as $\SubRtree{v}[u]$ (obviously $\SubRtree{\Root}=\Rtree$), and its size as $\Nvert{u}{v}=|V(\SubRtree{v}[u])|$ (Figure \ref{fig:theory:preliminaries:tree_types}). We call this {\em directional} size \citep{Hochberg2003a,Alemany2022a}. Note that $\Nvert{v}{u} + \Nvert{u}{v} = n$ for any $uv\in E$.

\begin{figure}
	\centering
\scalebox{1}{
\begin{tikzpicture}[ipe stylesheet]
  \pic
     at (60, 788) {ipe disk};
  \pic
     at (60, 760) {ipe disk};
  \pic
     at (76, 776) {ipe disk};
  \pic
     at (100, 772) {ipe disk};
  \pic
     at (88, 756) {ipe disk};
  \pic
     at (116, 756) {ipe disk};
  \pic
     at (124, 788) {ipe disk};
  \pic
     at (104, 788) {ipe disk};
  \pic
     at (124, 772) {ipe disk};
  \draw
    (60, 788)
     -- (76, 776)
     -- (100, 772)
     -- (124, 772);
  \draw
    (104, 788)
     -- (100, 772)
     -- (124, 788);
  \draw
    (116, 756)
     -- (100, 772);
  \draw
    (88, 756)
     -- (76, 776)
     -- (60, 760);
  \node[ipe node]
     at (56, 796) {$\Ftree$};
  \node[ipe node]
     at (76, 780) {$u$};
  \node[ipe node]
     at (92, 776) {$v$};
  \pic
     at (180, 788) {ipe disk};
  \pic
     at (160, 772) {ipe disk};
  \pic
     at (176, 772) {ipe disk};
  \pic
     at (192, 772) {ipe disk};
  \pic
     at (208, 772) {ipe disk};
  \node[ipe node]
     at (184, 792) {$u$};
  \node[ipe node]
     at (204, 776) {$v$};
  \draw[-{>[ipe arrow tiny]}]
    (180, 788)
     -- (161.665, 773.457);
  \draw[-{>[ipe arrow tiny]}]
    (180, 788)
     -- (176.568, 774.056);
  \draw[-{>[ipe arrow tiny]}]
    (180, 788)
     -- (190.886, 773.937);
  \draw[-{>[ipe arrow tiny]}]
    (180, 788)
     -- (206.186, 773.457);
  \pic
     at (192, 756) {ipe disk};
  \pic
     at (204, 756) {ipe disk};
  \pic
     at (216, 756) {ipe disk};
  \pic
     at (228, 756) {ipe disk};
  \draw[-{>[ipe arrow tiny]}]
    (208, 772)
     -- (193.289, 757.634);
  \draw[-{>[ipe arrow tiny]}]
    (208, 772)
     -- (204.436, 758.18);
  \draw[-{>[ipe arrow tiny]}]
    (208, 772)
     -- (215.349, 758.18);
  \draw[-{>[ipe arrow tiny]}]
    (208, 772)
     -- (226.574, 757.868);
  \node[ipe node]
     at (156, 796) {$\Rtree[u]$};
  \node[ipe node]
     at (40, 796) {a)};
  \node[ipe node]
     at (140, 796) {b)};
  \draw
    (180, 788) circle[radius=4];
\end{tikzpicture}
}
	\caption{a) A free tree $\Ftree$, where $\degree{u}=4$, and $\degree{v}=5$; in this tree, $\Nvert{u}{v}=5$ and $\Nvert{v}{u}=4$. b) The free tree $\Ftree$ rooted at $u$, denoted as $\Rtree[u]$, where $\outdegree{u}[u] = \outdegree{v}[\Rtree[u]] = \degree{u}=4$, and where $4 = \outdegree{v}[u] = \outdegree{v}[\Rtree[u]] < \degree{v}=5$. Figure borrowed from \citep{Hochberg2003a,Alemany2022a}.}
	\label{fig:theory:preliminaries:tree_types}
\end{figure}

As in previous research, we also decompose an edge $(\Root,u)$ in a projective arrangement $\arr$ into two parts: its {\em anchor} and its {\em coanchor}, as in Figure \ref{fig:theory:preliminaries:anchor_coanchor} \citep{Shiloach1979a,Chung1984a,Alemany2022b}. Informally, $\anchor{\Root u}$ is the number of vertices in $\arr$ covered by $(\Root,u)$ in the segment of $\SubRtree{u}$ including vertex $u$ (Figure \ref{fig:theory:preliminaries:anchor_coanchor}); similarly, $\coanchor{\Root u}$, is the number of vertices of $\arr$ covered by $(\Root,u)$ in segments that fall between $\Root$ and $u$ (Figure \ref{fig:theory:preliminaries:anchor_coanchor}). The length of an edge connecting $\Root$ with $u$ can be expressed with the formula
\begin{equation*}
\lenedge{\Root u} = |\arr(\Root) - \arr(u)| = \anchor{\Root u} + \coanchor{\Root u},
\end{equation*}
where $\anchor{\Root u}$ is the length of the anchor and $\coanchor{\Root u}$ is the length of the coanchor. The length of the anchor and coanchor can be formally defined as
\begin{align*}
\anchor{\Root u}   &= |\arr(u) - \arr(z)| + 1 \\
\coanchor{\Root u} &= |\arr(z) - \arr(\Root)| - 1,
\end{align*}
where $z\in V(\SubRtree{u})$ is the vertex of $\SubRtree{u}$ closest to $\Root$ in $\arr$ (Figure \ref{fig:theory:preliminaries:anchor_coanchor}). The same notation with $\arr$ omitted, $\Vanchor{\Root u}$ and $\Vcoanchor{\Root u}$ denote random variables. Furthermore, it will be useful to define the operator $\fixed$, which we use to condition expected values and constrain sets of arrangements of a rooted tree, in both cases to arrangements $\arr$ where (only) the root is fixed at the leftmost position of $\arr$. For instance, if $S$ is a set of arrangements $\arr$ of a rooted tree $\Rtree$ then $S^{\fixed}=\{ \arr \in S \;|\; \arr(\Root)=1 \}$. Moreover, if $X$ is defined on uniformly random arrangements from $S$ then $\expefixed{X}$ is the expected value of $X$ in uniformly random arrangements from $S^{\fixed}$.

\begin{figure}
	\centering
\scalebox{1}{
\thispagestyle{empty}
\noindent
\begin{tikzpicture}[ipe stylesheet]
  \pic
     at (288, 604) {ipe disk};
  \node[ipe node]
     at (286.384, 591.659) {$\Root$};
  \node[ipe node]
     at (154, 648) {$\anchor{\Root u}$};
  \draw
    (180, 636)
     -- (180, 640)
     -- (292, 640)
     -- (292, 636);
  \draw
    (160, 636)
     -- (160, 640)
     -- (180, 640)
     -- (180, 636);
  \draw
    (140, 608) rectangle (180, 600);
  \pic[ipe mark tiny]
     at (232, 604) {ipe disk};
  \pic[ipe mark tiny]
     at (224, 604) {ipe disk};
  \pic[ipe mark tiny]
     at (216, 604) {ipe disk};
  \node[ipe node]
     at (141.88, 588.221) {$\SubRtree{u}$};
  \draw
    (287.9998, 603.9999)
     arc[start angle=33.8902, end angle=142.4731, radius=77.6255];
  \node[ipe node]
     at (210, 648) {$\coanchor{\Root u}$};
  \draw
    (292, 608) rectangle (284, 600);
  \draw[ipe dash dashed]
    (180, 636)
     -- (180, 604);
  \draw[ipe dash dashed]
    (160, 636)
     -- (160, 604);
  \draw[ipe dash dashed]
    (292, 634)
     -- (292, 602);
  \draw
    (184, 608) rectangle (212, 600);
  \draw
    (236, 608) rectangle (280, 600);
  \node[ipe node]
     at (195.88, 588.221) {$\SubRtree{v}$};
  \node[ipe node]
     at (247.88, 588.221) {$\SubRtree{w}$};
  \draw
    (288, 604)
     arc[start angle=35.0345, end angle=139.7604, radius=55.6203];
  \draw
    (287.9998, 603.9999)
     arc[start angle=33.6209, end angle=132.1286, radius=21.2833];
  \pic[ipe mark small]
     at (178, 604) {ipe disk};
  \pic[ipe mark small]
     at (186, 604) {ipe disk};
  \pic[ipe mark small]
     at (162, 604) {ipe disk};
  \node[ipe node]
     at (158, 594) {$u$};
  \node[ipe node]
     at (174, 594) {$z$};
\end{tikzpicture}
}
	\caption{Illustration of an edge's anchor $\anchor{\Root u}$ and coanchor $\coanchor{\Root u}$. In this figure, $u,v,w\in\neighs{\Root}$. Figure adapted from \citep{Alemany2022b}.}
	\label{fig:theory:preliminaries:anchor_coanchor}
\end{figure}
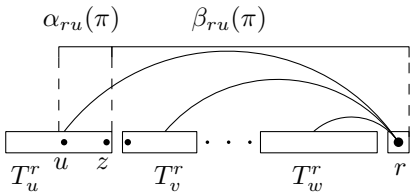

Finally, in this article we consider that two arrangements $\arr$ and $\arr'$ of the same tree $\Ftree$ are {\em different} if there is (at least) one vertex $u$ for which $\arr(u)\neq\arr'(u)$.
\subsection{Counting planar arrangements}
\label{sec:theory:counting_arrangements}

It is well known that the number of unconstrained arrangements of an $n$-vertex tree is $n!$. This is true given that arrangements are simply permutations, and unconstrained arrangements are not subject to any particular constraint, thus all vertex orderings are possible. Building on the fact that projective arrangements span over contiguous intervals \citep{Kuhlmann2006a}, \citet{Alemany2022b} studied the expected value of the random variable $\VD{\Rtree}$ in such arrangements by defining, as usual, a set of {\em segments} $\Phi_u$ associated to each vertex $u$, consisting of the segments associated to the subtrees $\SubRtree{u_1},\dots,\SubRtree{u_p}$ and $u$. A {\em segment} of a rooted tree $\SubRtree{u}$ is a segment within the linear ordering containing all vertices of $\SubRtree{u}$, an interval of length $\Nvert{\Root}{u}$ whose starting and ending positions are unknown until the whole tree is fully linearized; thus, a segment is a movable set of vertices within the linear ordering \citep{Alemany2022b}. For a vertex $u$, the set $\Phi_u$ is constructed from vertex $u$'s segment and the segments of its children $\outneighs{u}=\{u_1,\dots,u_k\}$ (Figure \ref{fig:theory:counting_arrangements:segments_of_trees}). Decomposing every vertex and its segments from the root to the leaves linearizes $\Rtree$ into a projective arrangement (Figure \ref{fig:theory:counting_arrangements:segments_of_trees}). This characterization led to a straightforward derivation of the total amount of projective arrangements of a rooted tree $\Rtree$ (Table \ref{table:introduction:results_summary})
\begin{equation}
\label{eq:theory:counting_arrangements:amount_projective}
\NProj = \prod_{u\in V} (\outdegree{u} + 1)!.
\end{equation}

\begin{figure}
	\centering
\scalebox{1}{
\begin{tikzpicture}[ipe stylesheet]
  \pic
     at (152, 776) {ipe disk};
  \pic
     at (88, 740) {ipe disk};
  \pic
     at (160, 740) {ipe disk};
  \pic
     at (224, 740) {ipe disk};
  \pic
     at (56, 708) {ipe disk};
  \pic
     at (76, 708) {ipe disk};
  \pic
     at (116, 708) {ipe disk};
  \draw[-{ipe pointed[ipe arrow small]}]
    (152, 776)
     -- (89.9278, 740.996);
  \draw[-{ipe pointed[ipe arrow small]}]
    (152, 776)
     -- (159.427, 742.351);
  \draw[-{ipe pointed[ipe arrow small]}]
    (152, 776)
     -- (221.414, 740.874);
  \draw[-{ipe pointed[ipe arrow small]}]
    (88, 740)
     -- (57.1252, 710.607);
  \draw[-{ipe pointed[ipe arrow small]}]
    (88, 740)
     -- (76.597, 710.854);
  \draw[-{ipe pointed[ipe arrow small]}]
    (88, 740)
     -- (114.555, 710.73);
  \pic[ipe mark tiny]
     at (88, 688) {ipe disk};
  \pic[ipe mark tiny]
     at (96, 688) {ipe disk};
  \pic[ipe mark tiny]
     at (104, 688) {ipe disk};
  \pic[ipe mark tiny]
     at (180, 708) {ipe disk};
  \pic[ipe mark tiny]
     at (192, 708) {ipe disk};
  \pic[ipe mark tiny]
     at (204, 708) {ipe disk};
  \draw
    (160, 740)
     -- (136, 664)
     -- (184, 664)
     -- (160, 740);
  \draw
    (224, 740)
     -- (200, 664)
     -- (248, 664)
     -- (224, 740);
  \draw
    (56, 708)
     -- (48, 680)
     -- (64, 680)
     -- (56, 708);
  \draw
    (76, 708)
     -- (68, 680)
     -- (84, 680)
     -- (76, 708);
  \draw
    (116, 708)
     -- (108, 680)
     -- (124, 680)
     -- (116, 708);
  \node[ipe node]
     at (72, 740) {$\Root_1$};
  \node[ipe node]
     at (144, 780) {$\Root$};
  \node[ipe node]
     at (144, 740) {$\Root_2$};
  \node[ipe node]
     at (204, 740) {$\Root_p$};
  \node[ipe node]
     at (60, 704) {$u_1$};
  \node[ipe node]
     at (80, 704) {$u_2$};
  \node[ipe node]
     at (104, 704) {$u_q$};
  \node[ipe node]
     at (48, 776) {a)};
  \node[ipe node]
     at (48, 624) {b)};
  \draw
    (80, 600) rectangle (112, 592);
  \draw
    (120, 600) rectangle (152, 592);
  \pic[ipe mark tiny]
     at (176, 596) {ipe disk};
  \pic[ipe mark tiny]
     at (184, 596) {ipe disk};
  \pic[ipe mark tiny]
     at (192, 596) {ipe disk};
  \pic
     at (164, 596) {ipe disk};
  \draw
    (200, 600) rectangle (232, 592);
  \node[ipe node]
     at (80, 652) {$\SubRtree{\Root_1}$};
  \node[ipe node]
     at (152, 652) {$\SubRtree{\Root_2}$};
  \node[ipe node]
     at (220, 652) {$\SubRtree{\Root_p}$};
  \node[ipe node]
     at (48, 672) {$\SubRtree{u_1}$};
  \node[ipe node]
     at (68, 672) {$\SubRtree{u_2}$};
  \node[ipe node]
     at (108, 672) {$\SubRtree{u_q}$};
  \node[ipe node]
     at (88, 580) {$\SubRtree{\Root_1}$};
  \node[ipe node]
     at (132, 580) {$\SubRtree{\Root_2}$};
  \node[ipe node]
     at (208, 580) {$\SubRtree{\Root_p}$};
  \node[ipe node]
     at (164, 580) {$\Root$};
  \draw[ipe pen heavier, ipe dash dash dotted, -{ipe pointed[ipe arrow small]}]
    (163.9996, 596)
     arc[start angle=0, end angle=173.267, radius=34.1176];
  \draw[-{ipe pointed[ipe arrow small]}]
    (164, 596)
     arc[start angle=180, end angle=351.2025, x radius=26.154, y radius=-26.154];
  \node[ipe node]
     at (48, 544) {c)};
  \node[ipe node]
     at (88, 500) {$\SubRtree{u_1}$};
  \node[ipe node]
     at (144, 500) {$\SubRtree{u_2}$};
  \node[ipe node]
     at (208, 500) {$\SubRtree{u_q}$};
  \draw
    (80, 520) rectangle (112, 512);
  \draw
    (136, 520) rectangle (168, 512);
  \pic[ipe mark tiny]
     at (176, 516) {ipe disk};
  \pic[ipe mark tiny]
     at (184, 516) {ipe disk};
  \pic[ipe mark tiny]
     at (192, 516) {ipe disk};
  \pic
     at (124, 516) {ipe disk};
  \draw
    (200, 520) rectangle (232, 512);
  \node[ipe node]
     at (120, 500) {$\Root_1$};
  \draw[-{ipe pointed[ipe arrow small]}]
    (124.0002, 516)
     arc[start angle=-9.5256, end angle=173.2653, radius=14.1463];
  \draw[-{ipe pointed[ipe arrow small]}]
    (124, 516)
     .. controls (128, 532) and (146, 532) .. (152, 520);
  \node[ipe node]
     at (60.123, 594.151) {$\Phi_{\Root}$ :};
  \node[ipe node]
     at (56.123, 514.275) {$\Phi_{\Root_1}$ :};
  \draw[ipe dash dotted]
    (80, 592)
     -- (80, 520);
  \draw[ipe dash dotted]
    (112, 592)
     -- (232, 520);
  \draw[-{ipe pointed[ipe arrow small]}]
    (124, 516)
     .. controls (128, 552) and (184, 564) .. (216, 520);
  \draw
    (160, 600) rectangle (168, 592);
  \draw
    (120, 520) rectangle (128, 512);
  \draw
    (152, 776) circle[radius=4];
  \draw
    (88, 764)
     .. controls (58.6667, 737.3333) and (42.6667, 704) .. (40, 664);
  \draw
    (88, 764)
     .. controls (117.3333, 737.3333) and (132, 704) .. (132, 664);
  \draw
    (40, 664)
     -- (132, 664);
  \draw[ipe pen heavier, ipe dash dash dotted, -{ipe pointed[ipe arrow small]}]
    (154.939, 566.419)
     arc[start angle=125.9654, end angle=163.6128, radius=81.8104];
  \draw[-{ipe pointed[ipe arrow small]}]
    (164, 596)
     arc[start angle=0.9602, end angle=162.7796, radius=14.322];
\end{tikzpicture}
}
	\caption{a) A rooted tree $\Rtree$ where $\neighs{\Root}=\{\Root_1, \dots, \Root_p\}$ are the $p$ children of $\Root$. The subtree $\SubRtree{\Root_1}$ has been circled for clarity. b) An example of a permutation of the segments in $\Phi_\Root$ associated to the root. c) An example of a permutation of the segments in $\Phi_{\Root_1}$ associated to $\Root_1$, the segment at the leftmost position in the example in (b). The dash-dotted edge in (b) and in (c) represent the same edge of the tree. In (b) and (c), respectively, $\Root$ and $\Root_1$ are segments of length 1.}
	\label{fig:theory:counting_arrangements:segments_of_trees}
\end{figure}
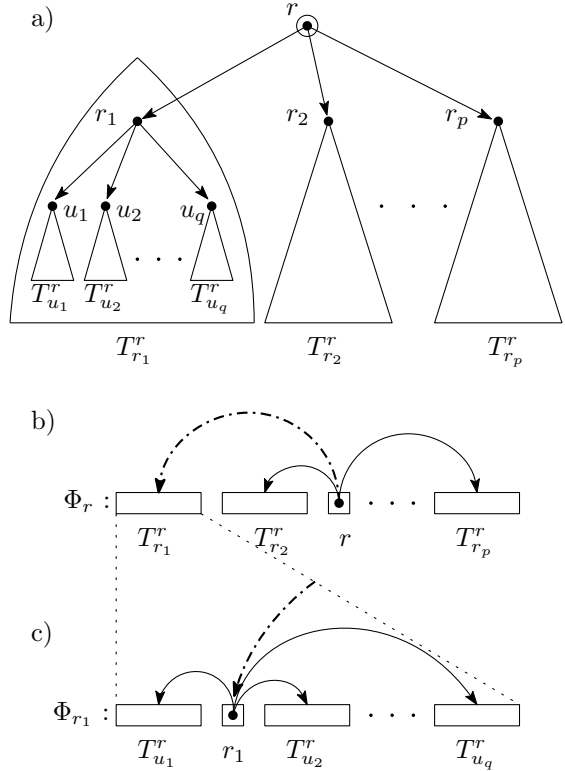

Using the structure of segments summarized above, we present a characterization of planar arrangements of free trees which helps to devise a method to generate planar arrangements u.a.r. (Section \ref{sec:theory:generating_arrangements:planar}) and to prove Theorem \ref{thm:introduction:E_pl_D} (Section \ref{sec:theory:E_pl_D}). To this aim, we define $\ProjFixed{\Root}$ as the set of projective arrangements of a rooted tree $\Rtree$ such that $\arr(\Root)=1$, and denote its size as $\NProjFixed{\Root}=|\ProjFixed{\Root}|$. Notice that when a vertex $u$ is fixed to the leftmost position, the planar arrangements in $\ProjFixed{u}$ are obtained by arranging the subtrees $\SubRtree{v}[u]$, $v\in\neighs{u}$, projectively to the right of $u$ in the linear arrangement. It is important to bear in mind that the operator $\fixed$ only fixes the root vertex $\Root$ to the leftmost position of the arrangement: the other vertices can be placed freely as long as the result is projective.

\begin{proposition}
\label{prop:theory:counting_arrangements:amount_planar}
The number of planar arrangements of an $n$-vertex free tree $\Ftree=(V,E)$, with $V=\{u_1,\cdots,u_n\}$ is
\begin{equation}
\label{eq:theory:counting_arrangements:amount_planar}
\NPlan
	= n\NProjFixed{u_1}
	= \cdots
	= n\NProjFixed{u_n}
	= n\prod_{u\in V} \degree{u}!.
\end{equation}
\end{proposition}

%
\begin{proof}
Given a free tree $\Ftree$, and any two distinct vertices $u,v$, it holds that $\ProjFixed{u}\cap \ProjFixed{v}=\emptyset$ because the vertices in the first positions are different. This lets us partition $\Plan$ into the non-empty pairwise-disjoint sets $\ProjFixed{u}$ and see that
\begin{equation*}
\NPlan = \sum_{u\in V} \NProjFixed{u}.
\end{equation*}
It is easy to see that
\begin{equation*}
\NProjFixed{u}
	= \degree{u}! \prod_{v\in\neighs{u}} \NProj[\SubRtree{v}[u]]
	= \prod_{v\in V} \degree{v}!.
\end{equation*}
We used Equation \ref{eq:theory:counting_arrangements:amount_projective} in the second equality. Notice that
\begin{equation*}
\NProjFixed{u_1} = \dots = \NProjFixed{u_n},
\end{equation*}
since the value $\NProjFixed{u}$ does not depend on the root vertex $u$. Therefore, Equation \ref{eq:theory:counting_arrangements:amount_planar} follows immediately.
\end{proof}
%

Obviously, there are more planar arrangements of a free tree $\Ftree$ than projective arrangements of any `rooting' $\Rtree$ of $\Ftree$, formally $\NPlan\ge\NProj$. We can see this by noticing that, when given a `rooting' of $\Ftree$ at $\Root\in V$,
\begin{equation*}
\frac{\NPlan}{\NProj}
	=
	\frac{
		n \degree{\Root}!\prod_{u\in V\setminus\{\Root\}}\degree{u}! 
	}{
		  (\degree{\Root} + 1)!\prod_{u\in V\setminus\{\Root\}}\degree{u}! 
	}
	=
	\frac{n}{\degree{\Root} + 1} \ge 1,
\end{equation*}
with equality when $\Ftree$ is a star tree\footnote{An $n$-vertex star tree consists of a vertex connected to $n-1$ leaves; it is also a complete bipartite graph $K_{1,n-1}$.} and $\Root$ is its vertex of highest degree.
\subsection{Generating arrangements uniformly at random}
\label{sec:theory:generating_arrangements}

Arrangements can be generated freely, that is, by imposing no constraint on the possible orderings, where all the $n!$ possible orderings are equally likely, or by imposing some constraint on the possible orderings. Generating unconstrained arrangements is straightforward: it is well known that a permutation of $n$ elements can be generated u.a.r. in time $\bigO{n}$ \citep{Cormen2001a}. It can be done as follows. Assume we are given a set of $n$ vertices, say $V=\{u_1,\dots,u_n\}$, and let $i=1$. Repeat the following steps $n$ times,
\begin{enumerate}
\item Select u.a.r. a vertex from $V$; the vertex is chosen with probability $1/(n-i+1)$. Let $u_i$ be said vertex,
\item Place $u_i$ in the arrangement at position $i$, that is, let $\arr(u_i)=i$,
\item Remove $u_i$ from $V$,
\item Increment $i$ by $1$.
\end{enumerate}
The product of all probabilities of vertex choice gives that the probability of producing a certain linear arrangement is
\begin{equation*}
\prod_{i=1}^n \frac{1}{n - i + 1} = \frac{1}{n!}
\end{equation*}
thus the arrangement is constructed uniformly at random. Since the removal of a vertex from the set and uniformly random choice of vertex can both be implemented in constant time (using arrays), the running time is $\bigO{n}$.

When constraints are involved, projectivity is often the preferred choice \citep{Gildea2007a,Liu2008a,Futrell2015a}. First, we present a $\bigO{n}$-time procedure to generate projective arrangements u.a.r. (Section \ref{sec:theory:generating_arrangements:projective}) and review methods used in past research (Section \ref{sec:theory:generating_arrangements:past_research}). Then we present a novel $\bigO{n}$-time procedure to generate planar arrangements u.a.r. (Section \ref{sec:theory:generating_arrangements:planar}) which in turn involves the generation of random projective arrangements of a subtree.

\subsubsection{Generating projective arrangements}
\label{sec:theory:generating_arrangements:projective}

The method we will present in detail here was outlined first by \citet{Futrell2015a}. Here we borrow from recent theoretical research summarized above \citep{Alemany2022b} to derive a detailed algorithm to generate projective arrangements and prove its correctness. 

{\small
\begin{algorithm}
	\caption{Generating projective arrangements u.a.r.}
	\label{algo:theory:generating_arrangements:projective:uar}
	\DontPrintSemicolon
	
	\SetKwProg{Fn}{Function}{ is}{end}
	\Fn{\textsc{Random\_Projective\_Arrangement}$(\Rtree)$} {
		\KwIn{$\Rtree$ a rooted tree.}
		\KwOut{A projective arrangement $\arr$ of $\Rtree$ chosen u.a.r.}
		$\arr \gets$ empty $n$-vertex arrangement\;
		\tcp{Algorithm \ref{algo:theory:generating_arrangements:projective:uar:subtree}}
		$\textsc{Random\_Projective\_Arrangement\_Subtree}(\Rtree, \Root, 1, \arr)$ \;
		\Return $\arr$
	}
\end{algorithm}
}

{\small
\begin{algorithm}
	\caption{Generating projective arrangements u.a.r. of a subtree.}
	\label{algo:theory:generating_arrangements:projective:uar:subtree}
	\DontPrintSemicolon
	
	\SetKwProg{Fn}{Function}{ is}{end}
	\Fn{\textsc{Random\_Projective\_Arrangement\_Subtree}$(\Rtree, u, p, \arr)$} {
		\KwIn{$\Rtree$ a rooted tree, $u$ any vertex of $\Rtree$, $p$ the starting position to arrange the vertices of $\SubRtree{u}$, $\arr$ partially-constructed without $\SubRtree{u}$.}
		\KwOut{$\arr$ partially-constructed with $\SubRtree{u}$.}
		
		$\Phi_u \gets$ a random permutation of $\outneighs{u}\cup\{u\}$ \;
		
		\For {$v\in\Phi_u$} {
			\If { $v = u$ } {
				$\arr(v)\gets p$\;
				$p\gets p + 1$
			}
			\Else {
				$\textsc{Random\_Projective\_Arrangement\_Subtree}(\Rtree, v, p, \arr)$\;
				$p\gets p + \Nvert{\Root}{v}$
			}
		}
	}
\end{algorithm}
}

In order to generate projective arrangements u.a.r., simply make random permutations of a vertex $u$ and its children $\outneighs{u}$, that is, choose one of the possible $(\outdegree{u} + 1)!$ permutations u.a.r. Algorithm \ref{algo:theory:generating_arrangements:projective:uar} formalizes this brief description.
The proof that Algorithm \ref{algo:theory:generating_arrangements:projective:uar} produces projective arrangements of a rooted tree $\Rtree$ u.a.r. is simple. The first call takes the root and its dependents and produces a uniformly random permutation with probability $1/(\degree{\Root} + 1)!$. Subsequent recursive calls (in Algorithm \ref{algo:theory:generating_arrangements:projective:uar:subtree}) produce the corresponding permutations each with its respective uniform probability, hence the probability of producing a particular permutation is the product of individual probabilities. Using Equation \ref{eq:theory:counting_arrangements:amount_projective}, we easily obtain that the probability of producing a certain projective arrangement is
\begin{equation*}
\prod_{u\in V} \frac{1}{(\outdegree{u} + 1)!} = \frac{1}{\NProj}.
\end{equation*}

\subsubsection{Generation of projective arrangements in past research}
\label{sec:theory:generating_arrangements:past_research}

Algorithm \ref{algo:theory:generating_arrangements:projective:uar} is equivalent to the ``fully random'' method used by \citet{Futrell2015a} as witnessed by the implementation of their code available on Github\footnote{\url{https://github.com/Futrell/cliqs/tree/44bfcf2c42c848243c264722b5eccdffec0ede6a}}, in particular in file {\tt cliqs/mindep.py}\footnote{\url{https://github.com/Futrell/cliqs/blob/44bfcf2c42c848243c264722b5eccdffec0ede6a/cliqs/mindep.py}} (function {\tt \_randlin\_projective}). Notice that \citet{Futrell2015a} outline (though vaguely) that a projective arrangement is generated randomly by ``Starting at the root node of a dependency tree, collecting the head word and its dependents and order them randomly''.

\citet{Futrell2015a} present their method to generate random projective arrangements as though it were the same as that by \citet{Gildea2007a,Gildea2010a}, who introduced a method to generate random linearizations of a tree which consists of ``choosing a random branching direction for each dependent of each head,\footnote{That is, as explained by \citet{Temperley2018a}, ``choose a random assignment of each dependent to either the left or the right of its head.''} and -- in the case of multiple dependents on the same side -- randomly ordering them in relation to the head'' \citep{Gildea2010a}. However, \citet{Futrell2015a} do not actually implement Gildea \& Temperley's method as witnessed by their code. Critically, Gildea \& Temperley's method does not produce uniformly random linearizations as we show with a counterexample.

Consider a star tree rooted at its hub. Let $X$ be a random variable for the position of the root in a random projective linear arrangement ($1 \leq X \leq n$). We have $\prob{X = x} = 1/n$ for all $x \in [1,n]$, therefore $X$ follows a uniform distribution and hence $\expe{X} = (n+1)/2$ and $\var{X} = (n^2 - 1)/12$ \citep{Mitzenmacher2017a}. Let $X'$ be a random variable for the position of the root according to Gildea \& Temperley's method. It is easy to see that $X'-1$ follows a binomial distribution with parameters $n-1$ and $1/2$. Namely, $\prob{X'- 1 = x} = {n - 1 \choose x}/2^{n-1}$. We have that $\expe{X'} = 1 + \expe{X'-1} = (n + 1)/2 = \expe{X}$, but $\var{X'} = \var{X' - 1} = (n - 1)/4$. Therefore, the variance in a truly uniformly random projective linear arrangement is $\Theta(n^2)$ while Gildea \& Temperley's method results in $\Theta(n)$, a much smaller dispersion. As $n\rightarrow \infty$, $X'-1$ converges to a Gaussian distribution. 

Gildea \& Temperley's method was introduced as a random baseline for the distance between syntactically-related words in languages and has been used with that purpose \citep{Gildea2007a,Gildea2010a,Temperley2018a}. Interestingly, the minimum baseline, namely, the minimum sum of dependency distances, results from placing the root at the center \citep{Shiloach1979a,Chung1984a}. The example above shows that Gildea \& Temperley's baseline tends to put the root at the center of the linear arrangement with higher probability than the truly uniform baseline. That behavior casts doubts on the power of that random baseline to investigate dependency distance minimization in languages since it tends to place the root at the center of the sentence, as expected from an optimal placement under projectivity \citep{Gildea2007a,Alemany2021a} and does it with much lower dispersion around the center than in truly uniformly random linearizations.

\subsubsection{Generating planar arrangements}
\label{sec:theory:generating_arrangements:planar}

Proposition \ref{prop:theory:counting_arrangements:amount_planar} leads to a method to generate planar arrangements u.a.r. for any free tree $\Ftree$. The method we propose is detailed in Algorithm \ref{algo:theory:generating_arrangements:planar:uar}.

{\small
\begin{algorithm}
	\caption{Generating planar arrangements u.a.r.}
	\label{algo:theory:generating_arrangements:planar:uar}
	\DontPrintSemicolon
	
	\SetKwProg{Fn}{Function}{ is}{end}
	\Fn{\textsc{Random\_Planar\_Arrangement}$(\Ftree)$} {
		\KwIn{$\Ftree$ a free tree.}
		\KwOut{A planar arrangement $\arr$ of $\Ftree$ chosen u.a.r.}
		
		$\arr \gets$ empty $n$-vertex arrangement\;
		$u\gets$ a vertex of $\Ftree$ chosen u.a.r. \;
		$\arr(u) \gets 1$\;
		
		$\Phi_u \gets$ a random permutation of $\neighs{u}$ \;
		
		$p\gets 2$ \;
		
		\For {$v\in\Phi_u$} {
			\tcp{Algorithm \ref{algo:theory:generating_arrangements:projective:uar:subtree}}
			$\textsc{Random\_Projective\_Arrangement\_Subtree}(\Rtree[u], v, p, \arr)$ \;
			$p\gets p + \Nvert{u}{v}$
		}
		\Return $\arr$
	}
\end{algorithm}
}

It is easy to see that Algorithm \ref{algo:theory:generating_arrangements:planar:uar} has time complexity $\bigO{n}$. Now we show that it generates planar arrangements uniformly at random. Firstly, choose a vertex, say $u\in V$, u.a.r., and place it at one of the arrangement's ends, say, the leftmost position; this vertex acts as a root for $\Ftree$. Secondly, choose u.a.r. one of the $\degree{u}!$ permutations of the segments of the subtrees $\SubRtree{v}[u]$ u.a.r. Lastly, recursively choose u.a.r. a projective linearization of every subtree $\SubRtree{v}[u]$ for $v\in\neighs{u}$ (Algorithm \ref{algo:theory:generating_arrangements:projective:uar:subtree}). These steps generate a planar arrangement u.a.r. since the probability of producing a certain planar arrangement following these steps is, then,
\begin{equation*}
\frac{1}{n}\frac{1}{\degree{u}!}\prod_{v\in \neighs{u}} \frac{1}{\NProj[\SubRtree{v}[u]]}
	= \frac{1}{n}\frac{1}{\degree{u}!}\prod_{v\in V\setminus\{u\}} \frac{1}{\degree{v}!}
	= \frac{1}{\NPlan}.
\end{equation*}
The equalities follow from Proposition \ref{prop:theory:counting_arrangements:amount_planar}.

\subsection{Expected sum of edge lengths}
\label{sec:theory:E_pl_D}

In this section we derive an arithmetic expression for $\ExpeDPlan$. First, we prove Theorem \ref{thm:introduction:E_pl_D}. To this aim, we define $\rexperfixed{\Vanchor{uv}}{\Root} = \rcondexpe{\Vanchor{uv}}{\arr(\Root)=1}$ as the expected value of $\Vanchor{uv}$ conditioned to the projective arrangements $\arr$ of $\Rtree$ such that $\arr(\Root)=1$; we define $\rexperfixed{\Vcoanchor{uv}}{\Root}$ likewise. The root is specified as a parameter of the expected value because we want to be able to use various roots. In the following proofs we rely heavily on Linearity of Expectation \citep[Theorem 2.1]{Mitzenmacher2017a} and the Law of Total Expectation \citep[Lemma 2.5]{Mitzenmacher2017a}.

%
\begin{proof}[Proof of Theorem \ref{thm:introduction:E_pl_D}]
We first prove Equation \ref{eq:introduction:E_pl_D:function_E_pr_D_fixed}. By the Law of Total Expectation,
\begin{equation*}
\ExpeDPlan = \sum_{u\in V} \condExpeDPlan{\arr(u)=1}\lprob{\arr(u)=1}.
\end{equation*}
Notice that, quite simply, that
\begin{equation*}
\condExpeDPlan{\arr(u)=1}
    = \condExpeDProj{\arr(u)=1}[\Rtree[u]]
    = \ExpeDProjFixedNoVertex[\Rtree[u]],
\end{equation*}
that is, the expected value of $D$ conditioned to planar arrangements of $\Ftree$ such that vertex $u$ is fixed at the leftmost position, $\condExpeDPlan{\arr(u)=1}$, is equal to the expected value of $D$ conditioned to projective arrangements of $\Rtree[u]$ such that vertex $u$ is fixed at the leftmost position, which is denoted as $\ExpeDProjFixedNoVertex[\Rtree[u]]$. By noticing, given a fixed vertex $u$, that $\lprob{\arr(u)=1} = \frac{1}{n}$, which is the proportion of planar arrangements of $\Ftree$ in which $\arr(u)=1$ (Proposition \ref{prop:theory:counting_arrangements:amount_planar}), Equation \ref{eq:introduction:E_pl_D:function_E_pr_D_fixed} follows immediately. Notice Equation \ref{eq:introduction:E_pl_D:function_E_pr_D_fixed} expresses the expected value of $D$ conditioned to planar arrangements of a free tree $\Ftree$ as the average of each of the expected values of $D$ conditioned to projective arrangements of $\Rtree[u]$ (for all $u\in V$) such that the root is fixed at the leftmost position.

Now we aim to write $\ExpeDProjFixedNoVertex[\Rtree[u]]$ as a function of $\ExpeDProj[\Rtree[u]]$. We start by decomposing $\ExpeDProjFixedNoVertex[\Rtree[u]]$ into a summation of expected values of the individual edge lengths, and group the edges of every subtree $\SubRtree{v}[u]$ of $\Rtree[u]$ (where $uv$ is a (directed) edge of the tree) into one single expected value for each subtree and leave the edges incident to the root $u$ in the same summation as follows
\begin{equation*}
\ExpeDProjFixedNoVertex[\Rtree[u]]
	=
	\sum_{vw\in \neighs{u}}
	\left(
		\rexpefixed{\Vd{vw}}{u} +
		\ExpeDProj[\SubRtree{v}[u]]
	\right).
\end{equation*}
Now, it is important to notice that we did not write $\ExpeDProjFixedNoVertex[\SubRtree{v}[u]]$ in the summation above since the conditioning imposed by the operator $\fixed$ in $\ExpeDProjFixedNoVertex[\Rtree[u]]$ only applies to the root $u$. The root of the subtrees can be placed freely in the arrangement as long as the result is projective. Now we decompose all (directed) edges $uv$ of $\Rtree$ in the first summation into anchor and coanchor, and we get
\begin{equation*}
\ExpeDProjFixedNoVertex[\Rtree[u]]
	=
	\sum_{v\in \neighs{u}}
	\left(
		\rexpefixed{\Vanchor{uv} + \Vcoanchor{uv}}{u} + \ExpeDProj[\SubRtree{v}[u]]
	\right).
\end{equation*}
Although the root $u$ is clear in this context, we have made it explicit in $\rexpefixed{\Vanchor{uv} + \Vcoanchor{uv}}{u}$ so as to be able to keep track of it in the following derivations. By linearity of expectation,
\begin{equation*}
\rexpefixed{\Vanchor{uv} + \Vcoanchor{uv}}{u} = \rexpefixed{\Vanchor{uv}}{u} + \rexpefixed{\Vcoanchor{uv}}{u}.
\end{equation*}
Now, notice that the length of the anchor of any given directed edge $(u,v)$, where $u$ is the head and $v$ is the dependent, is invariant to the position of $u$, that is, it only changes if we change the position of $v$ within its interval. Therefore, fixing the head to the leftmost position of the arrangement (or any position outside the segment of $v$) does not affect the value of $\rexpefixed{\Vanchor{uv}}{u}$ and we simply have that $\rexpefixed{\Vanchor{uv}}{u} = \rcondexpe{\Vanchor{uv}}{u}$ and thus
\begin{equation*}
\ExpeDProjFixedNoVertex[\Rtree[u]]
	=
	\sum_{v\in \neighs{u}}
	\left(
		\rcondexpe{\Vanchor{uv}}{u}
		+ \rexpefixed{\Vcoanchor{uv}}{u}
		+ \ExpeDProj[\SubRtree{v}[u]]
	\right).
\end{equation*}
The next step is to find the value of $\rexpefixed{\Vcoanchor{uv}}{u}$. Notice now that the length of the coanchor of any directed edge $(u,v)$ {\em is} affected by the position of the head $u$ and, as such, $\rexpefixed{\Vcoanchor{uv}}{u}$ need not be exactly equal to $\rcondexpe{\Vcoanchor{uv}}{u}$. The derivation is found in to the Appendix since it is merely an adaptation of the proof by \citet[Lemma 1]{Alemany2022b}; it gives
\begin{equation*}
\rexpefixed{\Vcoanchor{uv}}{u} = \frac{3}{2}\rcondexpe{\Vcoanchor{uv}}{u}.
\end{equation*}
Thus,
\begin{align}
\ExpeDProjFixedNoVertex[\Rtree[u]]
	&=
	\sum_{v\in \neighs{u}}
	\left(
		\rcondexpe{\Vanchor{uv}}{u} + \frac{3}{2}\rcondexpe{\Vcoanchor{uv}}{u} + \ExpeDProj[\SubRtree{v}[u]]
	\right) \nonumber\\
	&=
	\sum_{v\in \neighs{u}}
	\left(
		\rcondexpe{\Vd{uv}}{u} +
		\ExpeDProj[\SubRtree{v}[u]] +
		\frac{1}{2}\rcondexpe{\Vcoanchor{uv}}{u}
	\right) \nonumber\\
	&=
	\label{eq:theory:E_pl_D:formula:E_pr_beta_fixed}
	\ExpeDProj[\Rtree[u]] + \frac{1}{2} \sum_{v\in \neighs{u}} \rcondexpe{\Vcoanchor{uv}}{u}.
\end{align}
In the third equality we have used the identity by \citet[Equation 28]{Alemany2022b}, which states that in a rooted tree $\Rtree$
\begin{equation*}
\ExpeDProj =
	\sum_{v\in\neighs{\Root}} 
	\left(
		\rexpe{\Vd{\Root v}} +
		\ExpeDProj[\SubRtree{v}[\Root]]
	\right).
\end{equation*}
In this equation, we have not specified the expected values as being conditioned by the root $\Root$ since this is clear from the context. Plugging Equation \ref{eq:theory:E_pl_D:formula:E_pr_beta_fixed} into Equation \ref{eq:introduction:E_pl_D:function_E_pr_D_fixed} we get
\begin{equation}
\label{eq:theory:E_pl_D:formula:before_meaning_double_summation}
\ExpeDPlan
	=
	\frac{1}{2n}
	\sum_{u\in V}\sum_{v\in \neighs{u}} \rcondexpe{\Vcoanchor{uv}}{u}
	+
	\frac{1}{n}\sum_{u\in V} \ExpeDProj[\Rtree[u]].
\end{equation}
We can use the following result by \citet[Equation 16]{Alemany2022b}
\begin{equation*}
\rcondexpe{\Vcoanchor{uv}}{u}
    = \frac{\Nvert{u}{u} - \Nvert{u}{v} - 1}{3}
    = \frac{n - \Nvert{u}{v} - 1}{3}
\end{equation*}
to further simplify Equation \ref{eq:theory:E_pl_D:formula:before_meaning_double_summation} and, after proving that
\begin{align*}
\sum_{v\in \neighs{u}} \rcondexpe{\Vcoanchor{uv}}{u}
	&= \sum_{v\in \neighs{u}} \frac{\Nvert{u}{u} - \Nvert{u}{v} - 1}{3}
	= \frac{(n - 1)(\degree{u} - 1)}{3}, \\
\sum_{u\in V} \frac{1}{3}(n - 1)(\degree{u} - 1)
	&= \frac{(n - 1)(n - 2)}{3},
\end{align*}
we obtain
\begin{equation}
\label{eq:E_pl_D:formula:meaning_double_summation}
\frac{1}{2n}
	\sum_{u\in V} \sum_{v\in \neighs{u}} \rcondexpe{\Vcoanchor{uv}}{u}
	=
	\frac{(n - 1)(n - 2)}{6n}.
\end{equation}
Hence Equation \ref{eq:introduction:E_pl_D:function_E_pr_D}.
\end{proof}
%

For the sake of comprehensiveness, we also provide an arithmetic expression for the expected length of an edge $uv$ of a free tree in uniformly random planar arrangements. To this aim, we further define $\lexpefixed{\Vd{uv}}{\Root} = \lcondexpe{\Vd{uv}}{\arr(\Root)=1}$ to be the expected value of the length of edge $uv\in E(\Ftree)$ when the vertex $\Root\in V(\Ftree)$ is fixed to the leftmost position in planar arrangements of $\Ftree$. Similarly, given a rooting of $\Ftree$ at $\Root$, let $\rexpefixed{\Vd{uv}}{\Root} = \rcondexpe{\Vd{uv}}{\arr(\Root)=1}$ to be the expected value of the length of edge $uv\in E(\Rtree)$ when vertex $\Root$ acts as the root of the tree and it is fixed to the leftmost position in projective arrangements of $\Rtree$. The root vertex $\Root$ may be one of vertices $u$, $v$ or none of the two. In the expected value $\rexpefixed{\Vd{uv}}{\Root}$ we assume that the edge $uv$ is directed from $u$ to $v$ in accordance with the orientation defined by the root vertex $\Root$. Therefore, when $\Root$ is neither $u$ or $v$, the vertex of edge $uv$ closest to $\Root$ is always vertex $u$, and the farthest is always vertex $v$.

%
\begin{lemma}
\label{lemma:theory:E_pl_D:formula:expected_length_edge}
Given a free tree $\Ftree=(V,E)$, for any $uv\in E$ it holds that
\begin{equation}
\label{eq:theory:E_pl_D:formula:expected_length_edge}
\lexpe{\Vd{uv}} =
	1 +
	\frac{1}{n}
	\sum_{\Root\in V\setminus\{u,v\}}
		\rexper{\Vd{uv}}{\Root},
\end{equation}
where \citep{Alemany2022b}
\begin{equation}
\label{eq:theory:E_pl_D:formula:expected_length_edge:last_step:middle}
\rexper{\Vd{uv}}{\Root}
	= \frac{2\Nvert{\Root}{u} + \Nvert{\Root}{v} + 1}{6}.
\end{equation}
\end{lemma}
\begin{proof}
Following the characterization of planar arrangements described in Section \ref{sec:theory:counting_arrangements}, we have that $\lprob{\arr(\Root)=1}=1/n$. Then applying the Law of Total Expectation
\begin{equation}
\label{eq:theory:E_pl_D:formula:expected_length_edge:law_total_expectation}
\lexpe{\Vd{uv}}
	= \sum_{\Root\in V} \lcondexpe{\Vd{uv}}{\arr(\Root)=1}\lprob{\arr(\Root)=1}
	= \frac{1}{n}\sum_{\Root\in V} \lexpefixed{\Vd{uv}}{\Root}.
\end{equation}
Now we calculate $\lexpefixed{\Vd{uv}}{\Root}$ by cases. When $\Root\notin\{u,v\}$, 
\begin{equation}
\label{eq:theory:E_pl_D:formula:expected_length_edge:second_case}
\lexpefixed{\Vd{uv}}{\Root}
	= \rexpefixed{\Vd{uv}}{\Root}
	= \rcondexpe{\Vd{uv}}{\Root}.
\end{equation}
When $\Root\in\{u,v\}$, by linearity of expectation, 
\begin{equation*}
\lexpefixed{\Vd{uv}}{\Root}
	= \rexperfixed{\Vd{uv}}{\Root}
	= \rexperfixed{\Vanchor{uv} + \Vcoanchor{uv}}{\Root}
	= \rexperfixed{\Vanchor{uv}}{\Root} + \rexperfixed{\Vcoanchor{uv}}{\Root}.
\end{equation*}
By denoting $\overline{\Root}$ the only vertex in $\{u,v\}\setminus\{\Root\}$, then
\begin{equation}
\label{eq:theory:E_pl_D:formula:expected_length_edge:anchor:root_and_compl}
\rexperfixed{\Vanchor{uv}}{\Root}
	= \rexper{\Vanchor{uv}}{\Root}
	= \frac{\Nvert{\Root}{\overline{\Root}} + 1}{2}.
\end{equation}
Equation \ref{eq:theory:E_pl_D:formula:expected_length_edge:anchor:root_and_compl} relies on the fact that in a rooted tree $\Rtree$, the expected length of the anchor of an edge incident to the root, say $\Root w\in E(\Rtree)$, is given by $\rexper{\Vanchor{\Root w}}{\Root} = (\Nvert{\Root}{w} + 1)/2$ \citep{Alemany2022b}. An arithmetic expression for $\rexperfixed{\Vcoanchor{uv}}{\Root}$ can be found by modifying the proof of \citet[Lemma 1]{Alemany2022b}. Then, as before, we get (see Appendix),
\begin{equation}
\label{eq:theory:E_pl_D:formula:expected_length_edge:coanchor}
\rexperfixed{\Vcoanchor{uv}}{\Root}
	= \frac{3}{2}\rexper{\Vcoanchor{uv}}{\Root}
	= \frac{n - \Nvert{\Root}{\overline{\Root}} - 1}{2}.
\end{equation}
Therefore, by adding Equations \ref{eq:theory:E_pl_D:formula:expected_length_edge:anchor:root_and_compl} and \ref{eq:theory:E_pl_D:formula:expected_length_edge:coanchor} we obtain 
\begin{equation}
\label{eq:theory:E_pl_D:formula:expected_length_edge:first_case}
\lexpefixed{\Vd{uv}}{\Root}
	= \rexperfixed{\Vanchor{uv}}{\Root} + \rexperfixed{\Vcoanchor{uv}}{\Root}
	= \frac{\Nvert{\Root}{\overline{\Root}} + 1}{2} +
	  \frac{n - \Nvert{\Root}{\overline{\Root}} - 1}{2}
	= \frac{n}{2}.
\end{equation}
Equation \ref{eq:theory:E_pl_D:formula:expected_length_edge} follows immediately after inserting Equations \ref{eq:theory:E_pl_D:formula:expected_length_edge:first_case} and \ref{eq:theory:E_pl_D:formula:expected_length_edge:second_case} in  Equation \ref{eq:theory:E_pl_D:formula:expected_length_edge:law_total_expectation}.
\end{proof}
%

\section{Applications}
\label{sec:applications}

\subsection{A linear-time algorithm to compute $\ExpeDPlan$}
\label{sec:applications:algorithm}

Here we consider algorithms of increasing efficiency. First, since $\ExpeDProj[\Rtree[u]]$ can be calculated in $\bigO{n}$-time for any $n$-vertex rooted tree $\Rtree[u]$ \citep[Theorem 1]{Alemany2022b}, the evaluation `as is' of Equation \ref{eq:introduction:E_pl_D:function_E_pr_D} leads to an $\bigO{n^2}$-time algorithm.

Second, we could calculate the value $\ExpeDProj[\Rtree[u]]$ for all $u\in V$ in $\bigO{n}$-time and $\bigO{n}$-space with the following procedure:
\begin{enumerate}
\item Precompute $\Nvert{u}{v}$ in $\bigO{n}$-time \citep{Alemany2022a};
\item Choose an arbitrary vertex $w$;
\item Calculate $\ExpeDProj[\Rtree[w]]$ in $\bigO{n}$-time \citep{Alemany2022b}; and, finally,
\item Perform a Breadth First Search (BFS) traversal of $\Ftree$ starting at $w$. In this traversal, when going from vertex $u$ to vertex $v$, the value of $\ExpeDProj[\Rtree[v]]$ is calculated applying the precomputed value of $\ExpeDProj[\Rtree[u]]$ to Equation
\begin{equation*}
\ExpeDProj[\Rtree[u]] = \ExpeDProj[\Rtree[v]] + \Delta,
\end{equation*}
where $\Delta$ is equal to the difference $\ExpeDProj[\Rtree[u]] - \ExpeDProj[\Rtree[v]]$. We can obtain a formula for this difference by manipulating Equation \ref{eq:introduction:E_pr_D}. We get
\begin{align*}
\Delta
	&= \ExpeDProj[\Rtree[u]] - \ExpeDProj[\Rtree[v]] \\
	&= \frac{1}{6} \left[\Nvert{u}{v}\left(2\degree{v} - 1\right) 
		+ 2n\left(\degree{u} - \degree{v}\right) 
		- \Nvert{v}{u}\left(2\degree{u} - 1\right) \right].
\end{align*}
Notice that the value of $\Delta$ can be computed in constant time for any two vertices $u$ and $v$ (here we are interested in the value of $\Delta$ for pairs of adjacent vertices) and, crucially, without knowledge of either $\ExpeDProj[\Rtree[u]]$ or $\ExpeDProj[\Rtree[v]]$. That is, if the value of $\ExpeDProj[\Rtree[u]]$ is known then the value of $\ExpeDProj[\Rtree[v]]$ for any $v\in\neighs{u}$ can be calculated in constant time as
\begin{equation*}
\ExpeDProj[\Rtree[v]] = \ExpeDProj[\Rtree[u]] - \Delta.
\end{equation*}
\end{enumerate}

Third, we propose an alternative that is also $\bigO{n}$-time yet simpler and faster in practice, based on Proposition \ref{prop:applications:algorithm}.

%
\begin{proposition}
\label{prop:applications:algorithm}
Given a free tree $\Ftree=(V,E)$,
\begin{equation}
\label{eq:algorithm:linear_time_algorithm}
\ExpeDPlan
	=
	\frac{(n-1)(3n^2 + 2n - 2)}{6n}
	-
	\frac{1}{6n}
	\sum_{v\in V}
		(2\degree{v} - 1)
		\sum_{u\in \neighs{v}}\Nvert{v}{u}^2.
\end{equation}
\end{proposition}
\begin{proof}
Here we simplify the summation in Equation \ref{eq:introduction:E_pl_D:function_E_pr_D}, which becomes \citep{Alemany2022b} 
\begin{equation*}
\frac{1}{n}\sum_{u\in V}\ExpeDProj[\Rtree[u]]
	=
	\frac{1}{6n} \left( f(\Ftree) - n \right)
\end{equation*}
with 
\begin{equation*}
f(\Ftree) = \sum_{u\in V}\sum_{v\in V} \Nvert{u}{v}( \outdegree{v}[u] + 1 ).
\end{equation*}
Now we simplify $f(\Ftree)$ by first replacing the term $\outdegree{v}[u]$ by $\degree{v}$ after the necessary transformations so that we can swap the order of the summations afterwards, that is,
\begin{align}
f(\Ftree)
	&= 
	\sum_{u\in V}
	\left(
		\Nvert{u}{u}(2\outdegree{u}[u] + 1) + 
		\sum_{v\in V\setminus\{u\}} \Nvert{u}{v}( 2\outdegree{v}[u] + 1 )
	\right) \nonumber\\
	&= 
	\sum_{u\in V} n(2\degree{u} + 1) + 
	\sum_{u\in V}\sum_{v\in V\setminus\{u\}} \Nvert{u}{v}( 2\degree{v} - 1 ) \nonumber\\
	&= 
	n(5n - 4)
	- \sum_{u\in V} \Nvert{u}{u}( 2\degree{u} - 1 )
	+ 2\sum_{u\in V}\sum_{v\in V} \Nvert{u}{v}\degree{v} - \sum_{u\in V}\sum_{v\in V} \Nvert{u}{v} \nonumber\\
	&= 2n^2 + g(\Ftree) - h(\Ftree) \label{eq:applications:algorithm:introduction_g_h}
\end{align}
with
\begin{align}
\label{eq:applications:algorithm:derivation_E_pl_D:g}
	g(\Ftree) &= 2\sum_{u\in V}\sum_{v\in V} \Nvert{u}{v}\degree{v}, \\
\label{eq:applications:algorithm:derivation_E_pl_D:h}
	h(\Ftree) &= \sum_{u\in V}\sum_{v\in V} \Nvert{u}{v}.
\end{align}
In the preceding derivation, the second equality holds due to $\outdegree{v}[u]=\degree{v}-1$ for $v\neq u$; the third and fourth steps, we apply the Handshaking lemma.\footnote{The Handshaking lemma \citep{Gunderson2014a} states that the sum of the degrees of all vertices of a graph equals twice the number of its edges.} These lead to
\begin{equation}
\label{eq:applications:algorithm:derivation_E_pl_D:nice_El}
\frac{1}{n}\sum_{u\in V} \ExpeDProj[\Rtree[u]]
	=
	\frac{1}{6n}
	\left(
	n(2n - 1) + g(\Ftree) - h(\Ftree)
	\right).
\end{equation}
It remains to simplify Equations \ref{eq:applications:algorithm:derivation_E_pl_D:g} and \ref{eq:applications:algorithm:derivation_E_pl_D:h}. We start by changing the order of the summations in Equation \ref{eq:applications:algorithm:derivation_E_pl_D:g},
\begin{equation*}
g(\Ftree)
	= 2\sum_{v\in V}\sum_{u\in V} \Nvert{u}{v}\degree{v}
	= 2\sum_{v\in V}\degree{v}\sum_{u\in V} \Nvert{u}{v},
\end{equation*}
and continue simplifying the inner summation. Consider a fixed $v\in V$. We have that
\begin{equation*}
\underbrace{\sum_{u\in V} \Nvert{u}{v}}_{(1)}
	= n + \underbrace{\sum_{u\in V\setminus\{v\}} \Nvert{u}{v}}_{(2)}
	= n + \sum_{w\in \neighs{v}} \Nvert{w}{v}\Nvert{v}{w}.
\end{equation*}
The summation (1) adds up the size of all subtrees $\SubRtree{v}[w]$ with respect to a `moving' root $u$. In the first equality we have simply taken out the case $\Nvert{u}{u}$. To understand the second equality, focus for now on a single subtree $\SubRtree{w}[v]$ such that $wv\in E$. The summation (2) contains summands that correspond to all the vertices in $\SubRtree{w}[v]$, say vertices $u_1,\dots,u_k$ (assume, w.l.o.g., that $w=u_k$). These summands are $\Nvert{u_1}{v},\dots,\Nvert{u_k}{v}$ which are all equal to $\Nvert{w}{v}$ (Figure \ref{fig:applications:algorithm:summing_sizes_over_roots}). Moreover, there are $\Nvert{v}{w}$ vertices in $\SubRtree{w}[v]$ thus $k=\Nvert{v}{w}$, and this holds for all $w\in\neighs{v}$, hence the equality.
\begin{figure}
	\centering
\scalebox{1}{
\begin{tikzpicture}[ipe stylesheet]
  \draw
    (194.6667, 712.6667)
     .. controls (202.6667, 705.3333) and (217.3333, 698.6667) .. (230, 698)
     .. controls (242.6667, 697.3333) and (253.3333, 702.6667) .. (259.3333, 712.6667)
     .. controls (265.3333, 722.6667) and (266.6667, 737.3333) .. (260.6667, 746)
     .. controls (254.6667, 754.6667) and (241.3333, 757.3333) .. (231.3333, 755.3333)
     .. controls (221.3333, 753.3333) and (214.6667, 746.6667) .. (207.3333, 742.6667)
     .. controls (200, 738.6667) and (192, 737.3333) .. (188.6667, 732.6667)
     .. controls (185.3333, 728) and (186.6667, 720) .. cycle;
  \pic
     at (204, 724) {ipe disk};
  \pic
     at (228, 744) {ipe disk};
  \pic
     at (244, 720) {ipe disk};
  \node[ipe node]
     at (208, 728) {$u_1$};
  \node[ipe node]
     at (248, 724) {$u_3$};
  \node[ipe node]
     at (232, 748) {$u_2$};
  \pic
     at (232, 698) {ipe disk};
  \node[ipe node]
     at (228, 702) {$w$};
  \draw
    (232, 672)
     -- (232, 698);
  \pic
     at (232, 672) {ipe disk};
  \node[ipe node]
     at (236, 676) {$v$};
  \draw
    (232, 672)
     -- (204, 660);
  \draw
    (232, 672)
     -- (228, 648);
  \draw
    (232, 672)
     -- (248, 648);
  \draw
    (232, 672)
     -- (252, 664);
  \draw
    (268, 760)
     .. controls (272, 760) and (272, 758) .. (272, 756.3333)
     .. controls (272, 754.6667) and (272, 753.3333) .. (272, 750)
     .. controls (272, 746.6667) and (272, 741.3333) .. (272, 738)
     .. controls (272, 734.6667) and (272, 733.3333) .. (272, 732)
     .. controls (272, 730.6667) and (272, 729.3333) .. (272.6667, 728.6667)
     .. controls (273.3333, 728) and (274.6667, 728) .. (274.6667, 728)
     .. controls (274.6667, 728) and (273.3333, 728) .. (272.6667, 727.3333)
     .. controls (272, 726.6667) and (272, 725.3333) .. (272, 724)
     .. controls (272, 722.6667) and (272, 721.3333) .. (272, 718)
     .. controls (272, 714.6667) and (272, 709.3333) .. (272, 706)
     .. controls (272, 702.6667) and (272, 701.3333) .. (272, 699.6667)
     .. controls (272, 698) and (272, 696) .. (268, 696);
  \node[ipe node]
     at (280, 728) {$\Nvert{v}{w}$ vertices};
  \draw[ipe dash dashed]
    (268, 696)
     -- (184, 696);
  \draw[ipe dash dashed]
    (268, 760)
     -- (184, 760);
\end{tikzpicture}
}
	\caption{Proof of \ref{prop:applications:algorithm}. The value $\Nvert{u}{v}$ is the same for all vertices of $\SubRtree{w}[v]$ denoted as $\{u_1,\dots,u_k\}$ in the figure and the proof.}
	\label{fig:applications:algorithm:summing_sizes_over_roots}
\end{figure}
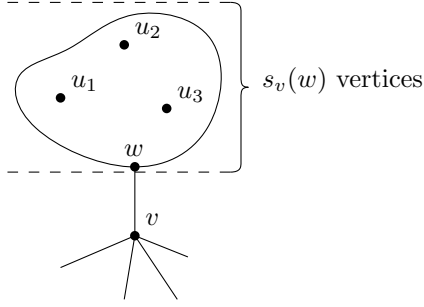
Finally,
\begin{equation}
\sum_{u\in V} \Nvert{u}{v}
	= n + \sum_{u\in \neighs{v}} (n - \Nvert{v}{u})\Nvert{v}{u}
	= n^2 - \sum_{u\in \neighs{v}}\Nvert{v}{u}^2,
	\label{eq:applications:algorithm:derivation_E_pl_D_linear:summation}
\end{equation}
thanks to the identity $\Nvert{u}{v}+\Nvert{v}{u}=n$. Then,
\begin{equation}
\label{eq:applications:algorithm:derivation_E_pl_D:g:final}
g(\Ftree)
	= 4n^2(n-1) - 2\sum_{v\in V}\degree{v}\sum_{u\in \neighs{v}}\Nvert{v}{u}^2.
\end{equation}
We use the result in Equation \ref{eq:applications:algorithm:derivation_E_pl_D_linear:summation} to simplify Equation \ref{eq:applications:algorithm:derivation_E_pl_D:h},
\begin{equation}
\label{eq:applications:algorithm:derivation_E_pl_D:h:final}
h(\Ftree)
	= \sum_{v\in V}\sum_{u\in V} \Nvert{u}{v}
	= n^3 - \sum_{v\in V}\sum_{u\in \neighs{v}}\Nvert{v}{u}^2.
\end{equation}
By combining Equations \ref{eq:applications:algorithm:derivation_E_pl_D:g:final} and \ref{eq:applications:algorithm:derivation_E_pl_D:h:final} into Equation \ref{eq:applications:algorithm:derivation_E_pl_D:nice_El} and, after some effort, we obtain
\begin{equation*}
\ExpeDPlan
	= \frac{(n-1)(n-2)}{6n}
	+ \frac{1}{6n}
	  \left(
		n(n-1)(3n+1)
		-
		\sum_{v\in V}(2\degree{v} - 1) \sum_{u\in \neighs{v}} \Nvert{v}{u}^2
	  \right)
\end{equation*}
which leads directly to Equation \ref{eq:algorithm:linear_time_algorithm}.
\end{proof}
%

%
\begin{lemma}
\label{lemma:applications:algorithm:linear_time_computable}
For any given free tree $\Ftree$, Algorithm \ref{algo:applications:algorithm} calculates $\ExpeDPlan$ in time and space $\bigO{n}$.
\end{lemma}
\begin{proof}
The pseudocode to calculate $\ExpeDPlan$ based on Proposition \ref{prop:applications:algorithm} is given in Algorithm \ref{algo:applications:algorithm}. This algorithm first calculates $\Nvert{u}{v}$ for all edges $uv\in E$, for the given tree $\Ftree$ in $\bigO{n}$ time using the pseudocode by \citet[Algorithm 2.1]{Alemany2022a}. Then it uses these values to calculate the sums of $\Nvert{v}{u}^2$ for every vertex $v\in V$. Such sums are then used to evaluate Equation \ref{eq:algorithm:linear_time_algorithm} hence calculating $\ExpeDPlan$ in time $\bigO{n}$. 
\end{proof}
%

{\small
\begin{algorithm}
	\caption{Calculation of $\ExpeDPlan$. Cost $\bigO{n}$-time, $\bigO{n}$-space.}
	\label{algo:applications:algorithm}
	\DontPrintSemicolon
	
	\SetKwProg{Fn}{Function}{ is}{end}
	\Fn{\textsc{compute\_expected\_planar}$(\Ftree)$} {
		\KwIn{$\Ftree$ free tree.}
		\KwOut{$\ExpeDPlan$.}
		\tcp{{\small \citealp[Algorithm 2.1]{Alemany2022a}}}
		$S\gets$\textsc{compute\_s\_ft}($\Ftree$) \;
		$L \gets \{0\}^n$ \tcp{{\small a vector of $n$ zeroes.}}
		\lFor {$( u,v, \Nvert{u}{v}) \in S$} {
			$L[u] \gets L[u] + \Nvert{u}{v}^2$
		}
		\Return $((n - 1)(3n^2 + 2n - 2) - \sum_{u\in V} (\degree{u} - 1)L[u])/6n$
	}
\end{algorithm}
}

\subsubsection{A simple application}
\label{sec:applications:algorithm:simple_extension}

Let $\ExpeDgek[1]$ be the expected value of the sum of edge lengths conditioned to arrangements $\arr$ such that $\Cr{\Ftree}\ge 1$. That is, arrangements such that the number of edge crossings is at least $1$. An immediate consequence of Lemma \ref{lemma:applications:algorithm:linear_time_computable} is that $\ExpeDgek[1]$ can be computed easily as the following corollary states.

\begin{corollary}
For any free tree $\Ftree$, $\ExpeDgek[1]$ can be computed in time and space $\bigO{n}$ thanks to the fact that
\begin{equation}
\label{eq:applications:algorithm:simple_extension:E_D_C_ge_1}
\ExpeDgek[1] = \frac{\ExpeDUnc - \ExpeDPlan\prob{\VCr{\Ftree} = 0}}{\prob{\VCr{\Ftree}\ge 1}}
\end{equation}
with $\prob{\VCr{\Ftree}\le 0} = \NPlan/n!$ and $\prob{\VCr{\Ftree}\ge 1} = (n! - \NPlan)/n!$.
\end{corollary}
\begin{proof}
Due to the Law of Total Expectation,
\begin{equation}
\label{eq:applications:algorithm:simple_extension:E_D_C_ge_1_aux}
\ExpeDUnc =
    \ExpeDPlan\prob{\VCr{\Ftree} = 0} +
    \ExpeDgek[1]\prob{\VCr{\Ftree}\ge 1},
\end{equation}
and hence Equation \ref{eq:applications:algorithm:simple_extension:E_D_C_ge_1}. $\NPlan$ can be computed in $\bigO{n}$-time with Equation \ref{eq:theory:counting_arrangements:amount_projective} and $\ExpeDPlan$ can be computed in time and space $\bigO{n}$ (Lemma \ref{lemma:applications:algorithm:linear_time_computable}). Hence all the components in the r.h.s. of Equation \ref{eq:applications:algorithm:simple_extension:E_D_C_ge_1} can be computed in time and space $\bigO{n}$.
\end{proof}

\subsection{Real syntactic dependency distances versus random baselines}
\label{sec:applications:corpora}

Evidence that dependency distances are smaller than expected by chance can be obtained by random baselines of varying strength 
\begin{itemize}
\item None, $\ExpeDUnc$, the expectation of $D(T)$ in unconstrained random linear arrangements \citep{Ferrer2004a},
\item Planarity, $\ExpeDPlan$, the expectation of $D(T)$ in planar random linear arrangements (this article),
\item Projectivity$, \ExpeDProj$, the expectation of $D(T)$ in projective random linear arrangements \citep{Gildea2007a,Alemany2022b}.
\end{itemize}
This raises the questions of what would the most appropriate baseline for research on dependency distance minimization be. $\ExpeDProj$ is by far the most widely used random baseline \citep{Gildea2007a,Liu2008a,Park2009a,Futrell2015a}. 

Since planarity is a weaker condition than projectivity, $\ExpeDPlan$ implies a gain in coverage. Accordingly, there are more planar sentences than projective sentences in real texts \citep[Table 1]{Havelka2007a,Gomez2010a} and also in artificially-generated syntactic dependency structures \citep[Figure 2]{Gomez2022a}. However, surprisingly, $\ExpeDPlan$ has never been used in research on the principle of dependency distance minimization. Here we aim to test the hypothesis that formal constraints mask the effects of the principle, a hypothesis that has already been confirmed on artificially-generated syntactic dependency structures \citep{Gomez2022a}.  

Since dependency distance naturally grows with sentence length \citep{Ferrer2014b,Ferrer2022a} and the manifestation of the principle depends on sentence length (the statistical bias towards shorter distances may disappear or become a bias in the opposite direction in short sentences \citealp{Ferrer2019b,Ferrer2022a}), we compare the actual dependency distances against the values predicted by the baselines in sentence of the same length. Given the natural growth of dependency distance as sentence length increases \citep{Ferrer2014b,Ferrer2022a}, we measure, for each sentence, the average dependency distance, namely $\left<d\right> = D(T)/(n - 1)$ instead of the raw total sum $D(T)$ (a sentence of $n$ vertices has $n - 1$ syntactic dependencies when the structure is a tree).   

\subsubsection{Data and methods}

As real datasets, we use the Parallel Universal Dependencies 2.6 collection \citep{UniversalDependencies26}. To control for annotation style, we consider two versions of the collection: the collection with its original content-head annotation (PUD) and its transformation into Surface-Syntactic Universal Dependencies 2.6 (hereafter PSUD). By doing so, we cover two major competing annotation styles \citep{Gerdes2018a}.

We borrow the preprocessing methods from previous research \citep{Ferrer2022a}. The main features of the processing is that nodes that are punctuation marks are removed and that the corpus remains fully parallel after the removal \citep{Ferrer2022a}. The preprocessed data is freely available as ancillary materials of the Linear Arrangement Library website.\footnote{Online at: \url{https://cqllab.upc.edu/lal/universal-dependencies/}}

With respect to previous accounts \citep{Havelka2007a, Gomez2010a, Ferrer2018a}, our collections exhibit some remarkable statistical differences. First, the proportion of projective and planar sentence is higher specially in PUD, where the proportion of non-projective or non-planar sentences does not exceed $10\%$ in most cases (Tables \ref{tab:application:coverage_PUD} and \ref{tab:application:coverage_PSUD}). This proportion increases in PSUD and in two exceptional languages, Chinese and Hindi, it becomes larger than $50\%$ (Tables \ref{tab:application:coverage_PSUD}). Second, the difference between the proportion of non-projective and non-planar sentences is smaller than in previous reports \citep{Gomez2010a,Havelka2007a}. Having said that, notice that our collections are fully parallel, and special care has been taken to keep annotation consistent across languages.

\begin{table}
	\centering
	\caption{Proportion (\%) of projective and planar sentences in the PUD collection. }
\begin{tabular}{c *{2}{d{3.3}} c *{2}{d{3.3}}}
\toprule
Language		& \mc{Projective}	& \mc{Planar}	& Language		& \mc{Projective}	& \mc{Planar} \\
\midrule
Arabic			& 96.2				& 96.3			& Italian		& 99.3				& 99.3 \\
Czech			& 89.6				& 89.8			& Japanese		& 99.7				& 99.7 \\
Chinese			& 99.4				& 99.4	 		& Korean		& 93.6				& 95.2 \\
German			& 86.3				& 86.7			& Polish		& 94.8				& 95.3 \\
English			& 95.5				& 95.9			& Portuguese	& 96.7				& 96.8 \\
Finnish			& 96.4				& 96.7			& Russian		& 97.6				& 98   \\
French			& 98.3				& 98.3			& Spanish		& 95.5				& 95.7 \\
Hindi			& 74.3				& 76.3			& Swedish		& 96.5				& 96.9 \\
Icelandic 		& 96.2				& 96.9			& Thai			& 97.2				& 97.2 \\
Indonesian		& 98.7				& 99			& Tukish		& 93.5				& 94.1 \\
\bottomrule
\end{tabular}
	\label{tab:application:coverage_PUD}
\end{table}
\begin{table}
	\centering
	\caption{Proportion (\%) of projective and planar sentences in the PSUD collection. }
\begin{tabular}{c *{2}{d{3.3}} c *{2}{d{3.3}}}
\toprule
Language		& \mc{Projective}	& \mc{Planar}	& Language		& \mc{Projective}	& \mc{Planar} \\
\midrule
Arabic			& 83.6				& 83.9			& Italian		& 94.5				& 94.6   \\
Czech			& 86.6				& 87.2			& Japanese		& 35.8				& 35.8   \\
Chinese			& 42				& 46.1			& Korean		& 75.8				& 77.1   \\
German			& 72.3				& 72.7			& Polish		& 88.2				& 89.7   \\
English			& 93.6				& 94.1			& Portuguese	& 87.3				& 87.7   \\
Finnish			& 88.8				& 89.4			& Russian		& 95.1				& 95.5   \\
French			& 90.5				& 90.6			& Spanish		& 80.2				& 80.9   \\
Hindi			& 43.6				& 44.3			& Swedish		& 93				& 93.7   \\
Icelandic		& 90.7				& 92			& Thai			& 85.6				& 86.8   \\
Indonesian		& 90.5				& 91.8			& Turkish		& 87.6				& 88.3   \\
\bottomrule
\end{tabular}
	\label{tab:application:coverage_PSUD}
\end{table}

Given formal constraint `*' (none, planarity and projectivity) and sentence length $n$, 
\begin{enumerate}
\item We calculate $\VD{\Rtree}$ for each $\Rtree$ and also calculate the expected sum of edge lengths under `*' different constraints (none, Equation \ref{eq:introduction:E_D}; planarity, Equation \ref{eq:introduction:E_pl_D:function_E_pr_D}; projectivity, Equation \ref{eq:introduction:E_pr_D}). 
\item Then, for each sentence, we divide each by $n-1$, to produce the mean length of its dependencies
\begin{equation*}
\langle d_* \rangle = \frac{\VD{}}{n - 1}
\end{equation*}
and the expected mean of length of its dependencies under some constraint `*'
\begin{equation*}
\expe{\langle d_* \rangle} = \frac{\eexpe{*}{\VD{}}}{n - 1}.
\end{equation*}

\item
Finally, we compute the average $\langle d_* \rangle$ and the average $\expe{\langle d_* \rangle}$ over all sentence of length $n$ satisfying constraint `*'.
\end{enumerate}

\subsubsection{Results}

\begin{figure}
	\centering
	\includegraphics[scale=0.88]{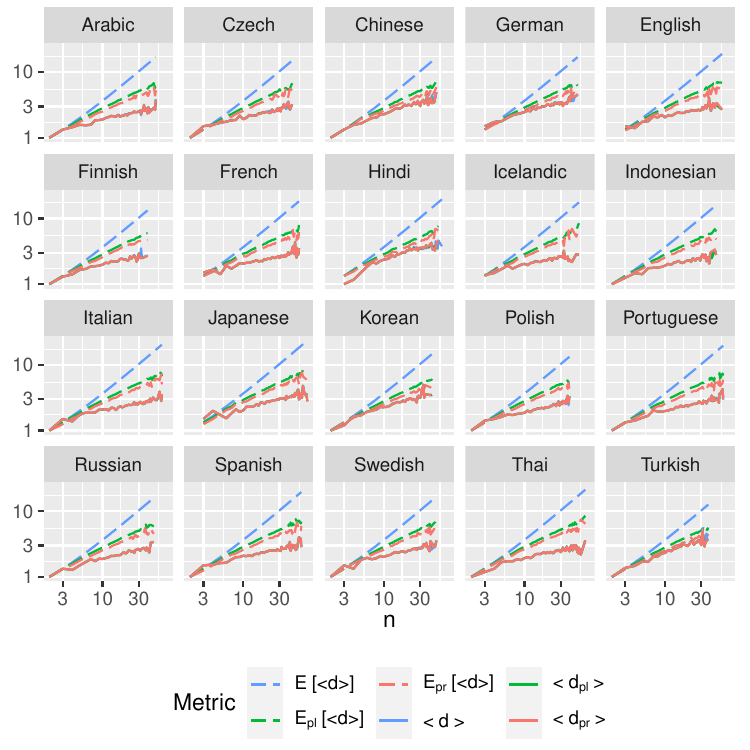}
	\caption{The scaling of $\langle d \rangle$, the mean dependency distance of a sentence as a function of sentence length ($n$) for languages in the PUD collection for formal constraints of increasing strength: none (blue), planarity (green) and projectivity (red). Lines indicate the average value over all sentences of the same length. Solid lines are used for real sentences and dashed lines are used for the corresponding random baseline. Solid lines overlap so much that only one of them can be seen in most cases.}
	\label{fig:application:baselines_PUD}
\end{figure}

\begin{figure}
	\centering
	\includegraphics[scale=0.88]{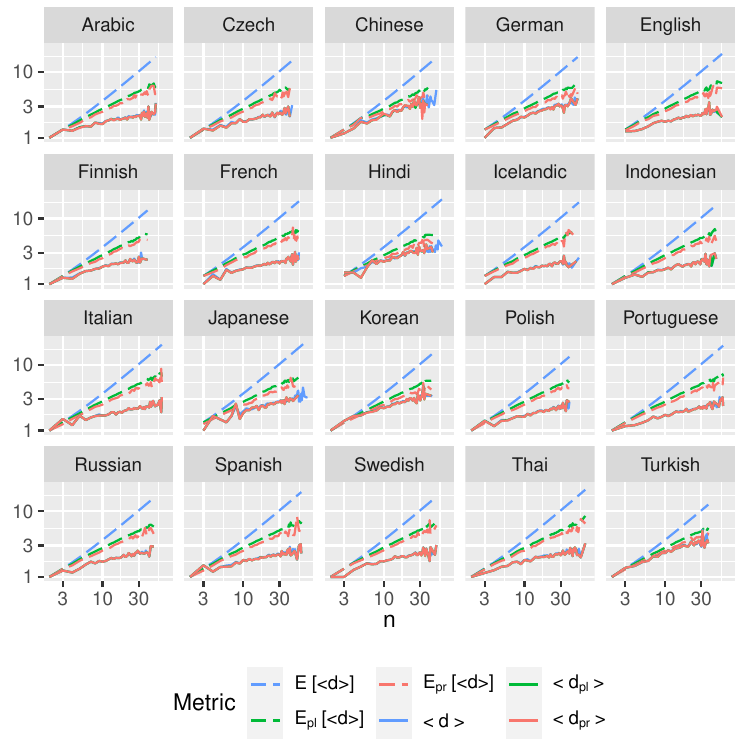}
	\caption{The scaling of $\langle d \rangle$, the mean dependency distance of a sentence as a function of sentence length ($n$) for languages in the PSUD collection for formal constraints of increasing strength. Format is the same as in Figure \ref{fig:application:baselines_PUD}. Again, solid lines overlap that only one of them can be seen in most cases.}
	\label{fig:application:baselines_PSUD}
\end{figure}

Figures \ref{fig:application:baselines_PUD} and \ref{fig:application:baselines_PSUD} show the scaling of mean dependency distance as a function of sentence length in real sentences and in their corresponding random baselines. Concerning the random baselines (dashed lines), we find that the stronger the formal constraint on syntactic dependency structures the lower the value of the random baseline. In contrast, the actual mean sentence length (solid lines) is practically the same independently of the formal constraint (none, planarity and projectivity). This is due to the fact the proportion of sentences that are lost by imposing some formal constraint is small in the PUD and PSUD collections. The overwhelming majority of sentences are planar and the proportion of planar sentences that are not projective is really small (Table \ref{tab:application:coverage_PUD} and \ref{tab:application:coverage_PSUD}). Thus, selecting sentences satisfying a certain formal constraint has a neglectable impact on the estimation of mean dependency distance.

Concerning the relationship between the actual mean dependency distance and the random baselines, we find that the average $\langle d \rangle$ is below the average value of the random baselines for sufficiently large $n$ in all languages. The only exception is Turkish, where the actual average $\langle d \rangle$ is just slightly below the average of the projective baseline (Figures \ref{fig:application:baselines_PUD} and \ref{fig:application:baselines_PSUD}).

These findings are consistent between PUD and PSUD, in spite of their differences in proportions of projective and planar sentences commented above.
\section{Conclusions and future work}
\label{sec:conclusions}

\subsection{Theory}

In Section \ref{sec:theory:counting_arrangements}, we have characterized planar arrangements of a given free tree $\Ftree$ using the concept of segment \citep{Alemany2022b}. Employing said characterization, we have shown that the number of planar arrangements of a free tree depends on its degree sequence (Proposition \ref{prop:theory:counting_arrangements:amount_planar}), in a similar way projective arrangements of a rooted tree do \citep{Alemany2022b}. Moreover, we have given a procedure to generate u.a.r. planar arrangements of a given free tree in Section \ref{sec:theory:generating_arrangements} (Algorithm \ref{algo:theory:generating_arrangements:planar:uar}) which can be easily adapted to generate such arrangements exhaustively. Interestingly, our algorithm to generate planar arrangements is based on the generation of projective arrangements of a rooted subtree. For the sake of completeness, we have detailed a procedure to generate u.a.r. projective arrangements of a given rooted tree (Algorithm \ref{algo:theory:generating_arrangements:projective:uar}). 

\subsection{Applications}

The identification of the underlying structure of planar arrangements have led us to derive an arithmetic expression, in Section \ref{sec:theory:E_pl_D}, for $\ExpeDPlan$ (Theorem \ref{thm:introduction:E_pl_D}) from which we devised a $\bigO{n}$-time algorithm to calculate such value (Proposition \ref{prop:theory:counting_arrangements:amount_planar}, Algorithm \ref{algo:applications:algorithm}). 

In Section \ref{sec:applications}, we have applied the theory developed so far to investigate the effect of formal constraints of increasing strength (none, planarity, projectivity) in a parallel collection and reported two main findings. First, the average dependency distance in real sentences remains practically the same as the strength of the formal constraint increases. We believe that this result stems from the high proportion of planar sentences (and the very low proportion of planar sentences that are not projective) of the PUD collection. Higher proportions of non-planar sentences have been reported in other collections \citep{Gomez2017a}. Second, the tendency of the random baseline to have a smaller value in stronger formal constraints. Critically, this phenomenon indicates that the strength of the dependency distance minimization effect depends on the choice of the formal constraint for the random baseline. As these formal constraints may be a side-effect of dependency distance minimization \citep{Ferrer2006a,Gomez2017a,Gomez2022a,Yadav2022a}, this phenomenon suggests that 
\begin{enumerate}
\item Formal constraints absorb the dependency distance effect. 
\item A fairer evaluation of the actual degree of optimization of dependency distances or a more accurate measurement of the power of the effect of dependency distance minimization requires considering not only the magnitude of the effect with respect some random baseline but also the formal constraint, as the latter may hide part of the dependency distance minimization effect.
\end{enumerate}
 
In past research on syntactic dependency distance minimization, $\ExpeDProj$ has been the most widely used random baseline \citep{Gildea2007a,Liu2008a,Park2009a,Futrell2015a}. However, projectivity has a lower coverage than planarity in real sentences \citep{Havelka2007a,Gomez2010a}. Projectivity is at risk of underestimating the strength of the dependency distance minimizaton principle \citep{Ferrer2004a} because of the significant reduction in the value of the random baseline (Figures \ref{fig:application:baselines_PUD} and \ref{fig:application:baselines_PSUD}) or the reduction of the actual dependency distances \citep[Figure 2]{Gomez2022a} that it introduces. Thanks to the research in this article, we have paved the way for replicating past research replacing $\ExpeDProj$ with $\ExpeDPlan$.

\subsection{Future work}

Planarity is a relaxation of projectivity but future work should address the problem of the expected value of $\VD{\Ftree}$ in classes of formal constraints with even more coverage \citep{Ferrer2018a}. A promising step is the investigation of $\ExpeDk$, the expected value of $\VD{\Ftree}$ conditioned to arrangements $\arr$ such that $\Cr{\Ftree}\le k$, that is, in arrangements such that the number of edge crossings is at most $k$. Notice that $\ExpeDk[0]=\ExpeDPlan$. In real languages, the average number of crossings ranges between $0.40$ and $0.62$ \citep{Ferrer2018a}, suggesting that $\ExpeDk$ with $k = 1$ or a small $k$ would suffice.


\appendix
\section{Derivation of $\rexpefixed{\Vcoanchor{uv}}{u}$}
\label{appendix:derivation_E_pr_beta_fixed}

Here we derive the expected length of the coanchor of a (directed) edge $uv\in E(\Rtree[u])$ in uniformly random projective arrangements of $\Rtree[u]$ conditioned to $\arr(u)=1$. Following \citet{Alemany2022b}, we decompose the length of the coanchor of the (directed) edge $uv$, $\Vcoanchor{uv}$, as the sum of the lengths of the segments in-between $u$ and $v$ (Figure \ref{fig:theory:preliminaries:anchor_coanchor}). Here we use $k_{uv}$ to denote the number of segments in-between $u$ and $v$, and $\Vlengthsegment{uv}{i}$ to denote the size of the $i$th segment, yielding \citep{Alemany2022b},
\begin{equation*}
\Vcoanchor{uv} = \sum_{i=1}^{k_{uv}} \Vlengthsegment{uv}{i}.
\end{equation*}
By the Law of Total Expectation, we have that
\begin{equation}
\label{eq:appendix:derivation_E_pr_beta_fixed:decomposition_into_segments}
\rexpefixed{\Vcoanchor{uv}}{u}
	=
	\sum_{k=1}^{\degree{u} - 1}
		\rexpefixed{\Vcoanchor{uv}}{u, k_{uv} = k}
		\rprobfixed{k_{uv} = k}{u},
\end{equation}
where $\rexpefixed{\Vcoanchor{uv}}{u, k_{uv} = k}$ is the expectation of $\Vcoanchor{uv}$ given that $u$ is the root of the tree (fixed at the leftmost position), and that $u$ and $v$ are separated by $k$ segments, and $\rprobfixed{k_{uv} = k}{u}$ is the probability that $u$ and $v$ are separated by $k$ intermediate segments, both in uniformly random projective arrangements $\arr$ conditioned to $\arr(u)=1$, both conditioned to the root of the tree being vertex $u$. On the one hand,
\begin{equation}
\label{eq:appendix:derivation_E_pr_beta_fixed:conditioned_value}
\rexpefixed{\Vcoanchor{uv}}{u, k_{uv} = k}
    = \rexpefixed{\sum_{i=1}^{k} \Vlengthsegment{uv}{i}}{u}
    = \frac{n - \Nvert{u}{v} - 1}{\degree{u} - 1} k.
\end{equation}
Notice that this is the same result as that obtained in \citep{Alemany2022b}. Lastly, the proportion of arrangements in which the segment of $v$ is at position $k_{uv}+1$ equals $(\degree{u} - 1)!$, therefore, 
\begin{equation}
\label{eq:appendix:derivation_E_pr_beta_fixed:probability_s_segments}
\rprobfixed{k_{uv} = k}{u}
    = \frac
        {(\degree{u} - 1)!	\prod_{v\in\neighs{u}} \NProj[\Rtree[u]] }
        {\degree{u}!		\prod_{v\in\neighs{u}} \NProj[\Rtree[u]] }
    = \frac{1}{\degree{u}}.
\end{equation}
Recalling that \citep{Alemany2022b}
\begin{equation*}
\rcondexpe{\Vcoanchor{uv}}{u} = \frac{\Nvert{u}{u} - \Nvert{u}{v} - 1}{3},
\end{equation*}
and plugging the results in Equations \ref{eq:appendix:derivation_E_pr_beta_fixed:conditioned_value} and \ref{eq:appendix:derivation_E_pr_beta_fixed:probability_s_segments} into Equation \ref{eq:appendix:derivation_E_pr_beta_fixed:decomposition_into_segments} we get
\begin{equation*}
\rexpefixed{\Vcoanchor{uv}}{u}
    = \frac{n - \Nvert{u}{v} - 1}{\degree{u} - 1} \frac{1}{\degree{u}} \sum_{k=1}^{\degree{u}-1} k
    = \frac{\Nvert{u}{u} - \Nvert{u}{v} - 1}{2}
    = \frac{3}{2}\rcondexpe{\Vcoanchor{uv}}{u}.
\end{equation*}



\end{document}